\documentclass{article}

\usepackage{arxiv}

\usepackage{microtype}
\usepackage{graphicx}
\usepackage{subcaption}
\usepackage{booktabs}
\usepackage{hyperref}
\usepackage{amsmath}
\usepackage{amssymb}
\usepackage{mathtools}
\usepackage{amsthm}
\usepackage[capitalize,noabbrev]{cleveref}
\usepackage{natbib}
\usepackage{tikz}
\usepackage{tikz-cd}
\usepackage{pgfplots}
\usepackage{xcolor}
\usepackage{enumitem}
\RequirePackage{doi}
\usetikzlibrary{arrows.meta,calc,positioning,decorations.pathreplacing,fit}
\pgfplotsset{compat=1.18}

\hypersetup{
  pdftitle={\textsc{Odyssey}: Constructing Verifiable, Local Truth-Preserving Foundation Models},
  pdfsubject={Inspectable foundation model construction},
  pdfauthor={Anonymous Authors},
  pdfkeywords={foundation models, Toulmin argumentation, sheaves, topos world models, foundries},
  colorlinks=true,
  linkcolor=blue!55!black,
  citecolor=blue!55!black,
  urlcolor=blue!55!black
}

\theoremstyle{plain}
\newtheorem{theorem}{Theorem}[section]
\newtheorem{proposition}[theorem]{Proposition}

\newtheorem{corollary}[theorem]{Corollary}
\theoremstyle{definition}
\newtheorem{definition}[theorem]{Definition}

\theoremstyle{remark}

\newcommand{\C}{\mathcal{C}}

\newcommand{\Psh}{\mathcal{P}}

\newcommand{\doop}{\operatorname{do}}

\title{\textsc{Odyssey}: Constructing Verifiable, Local Truth-Preserving Foundation Models  \thanks{Paper to be presented as a 2.5 hour tutorial at ICML 2026, July 6, Seoul, South Korea.}}

\author{ Sridhar Mahadevan \\
	Adobe Research and University of Massachusetts, Amherst\\
	\texttt{smahadev@adobe.com, mahadeva@umass.edu}
}

\date{}

\begin{document}
\maketitle

\begin{abstract}
We introduce a categorical framework for constructing verifiable,
local truth-preserving foundation models as compositions of \emph{foundries}:
building-block architectural components that specify a cover of local
contexts, local representation families, restriction maps, gluing rules,
obstruction policies, update obligations, and human-facing views.  The framework is
implemented using a system called \textsc{Odyssey}, where a foundry is not just
an organized sheaf of knowledge, but one that carries within it an
argumentation component for answering not only ``what does the foundry
contain?'' but also ``why is this claim warranted here, and where would the
argument fail?''  Concrete foundries are built from generic foundries
such as evidence/argument,
operational decision, institutional/financial, market meaning, scientific
challenge, research-program, assistant-build, and evaluation-harness foundries.
The resulting models are sheaf-like families of local predictive and logical
models whose consistency, provenance, promotion decisions, and failures to glue
remain explicit artifacts.  Foundry management is assigned to five agents:
\textsc{Scylla} for human-facing inspection, \textsc{Homer} for source
ingestion, \textsc{Athena} for construction and audit, and
\textsc{Prometheus} for causal and predictive state, while \textsc{Toulmin}
supplies the argumentative reasoning layer.  Foundry SQL (FSQL) provides a
typed query surface for maintained artifacts.  To turn external pretrained
models, including GPT-style models, into durable \textsc{Odyssey} foundries, we
introduce \textsc{Ticket}, Topos Integration using Causal Kan Extension
Transformers.  The ICML tutorial version also integrates BRIDGE/SKFM
causal-geometry screens for latent causal refinement
\citep{mahadevan2026latentconfoundedcausaldiscovery}, infinitesimal-causality
diagnostics for local causal variation
\citep{mahadevan2026infinitesimalcausality}, and SkillOpt optimization of
natural-language admission policies for causal claim tickets
\citep{yang2026skilloptexecutivestrategyselfevolving}.  We
formalize our specific implementation of foundry
construction and management as Universal Foundry Learning (UFL): the universal
property is expressed as a composition of left and right Kan extensions, with
left Kan extension rolling local artifacts into candidate foundries and right
Kan extension enforcing the restriction, gluing, obstruction, and argumentation
conditions required for promotion.
\textsc{Odyssey} is fully implemented and tested across a wide spectrum of concrete
foundries, showing that the same categorical machinery supports domain
construction, artifact replay, sheaf diagnostics, grounded Toulmin/local-LLM
scrutiny, BRIDGE/SKFM residual-obstruction ledgers, and optimized
TICKET-compatible causal-claim extraction across heterogeneous sources. This paper is to be presented as a 2.5 hour tutorial at ICML 2026. The tutorial home page is at \url{https://bit.ly/4ajS0nA}.   \end{abstract}


\section{Introduction}

Large language models are useful general-purpose interfaces, but their dominant
form of reuse is still a poor fit for many scientific, industrial, and
operational settings.  A user rarely wants only a larger embedding or a flat
summary.  They want a foundation model for a domain that can be designed,
inspected, adapted, refreshed, and trusted: a model of a retailer, a research
program, a corpus benchmark, a company filing workflow, a repair-manual
procedure, or a scientific controversy.  Such a model must preserve local
evidence, expose uncertainty, remember provenance, and say when a claim should
not be transported from one regime to another.

\textsc{Odyssey} is a framework for constructing such foundation models as
\emph{foundries}.  A foundry is a reusable model-making architecture: it
specifies a cover of local contexts, local representation families,
restriction maps, gluing rules, obstruction policies, update obligations, and
human-facing views.  Concrete foundries are built from generic foundries such
as evidence/argument, operational decision, institutional/financial, market
meaning, scientific challenge, research-program, assistant-build, and
evaluation-harness foundries.  Specialized foundries are typed compositions of
these generic objects, for example a sporting-goods retailer foundry that
combines review evidence, store operations, brand meaning, and corporate filing
evidence.

\begin{figure}[t]
\centering
\footnotesize
\begin{tikzpicture}[
  node distance=7mm,
  basis/.style={draw,rounded corners,align=center,text width=0.25\linewidth,inner sep=5pt},
  model/.style={draw,rounded corners,align=center,text width=0.78\linewidth,inner sep=6pt,fill=black!3},
  arrow/.style={-Latex,thick}
]
\node[basis] (evidence) {Evidence / argument};
\node[basis,right=of evidence] (ops) {Operational decision};
\node[basis,right=of ops] (market) {Market meaning};
\node[basis,below=of evidence] (finance) {Institutional / financial};
\node[basis,right=of finance] (science) {Scientific challenge};
\node[basis,right=of science] (eval) {Evaluation harness};
\node[model,below=12mm of science] (space) {Free foundry space: specialized foundation models are typed combinations and restrictions of generic foundries};
\node[model,below=of space] (examples) {Examples: DKS = retail/brand/review/filing; KET = research/evaluation over PTB, WikiText-2, WikiText-103; Amazon Reviews 2023 = corpus-benchmark; MyFixIt = procedural repair PSR; IKEA ASM = multimodal assembly PSR};
\draw[arrow] (evidence) -- (space);
\draw[arrow] (ops) -- (space);
\draw[arrow] (market) -- (space);
\draw[arrow] (finance) -- (space);
\draw[arrow] (science) -- (space);
\draw[arrow] (eval) -- (space);
\draw[arrow] (space) -- (examples);
\end{tikzpicture}
\caption{Foundries as reusable building blocks for foundation models.  The
``basis vector'' language is an algebraic analogy: generic foundries span a
free space of typed model-making objects, and concrete foundation models are
compositions, restrictions, and gluing decisions over those objects.}
\label{fig:foundry-basis}
\end{figure}
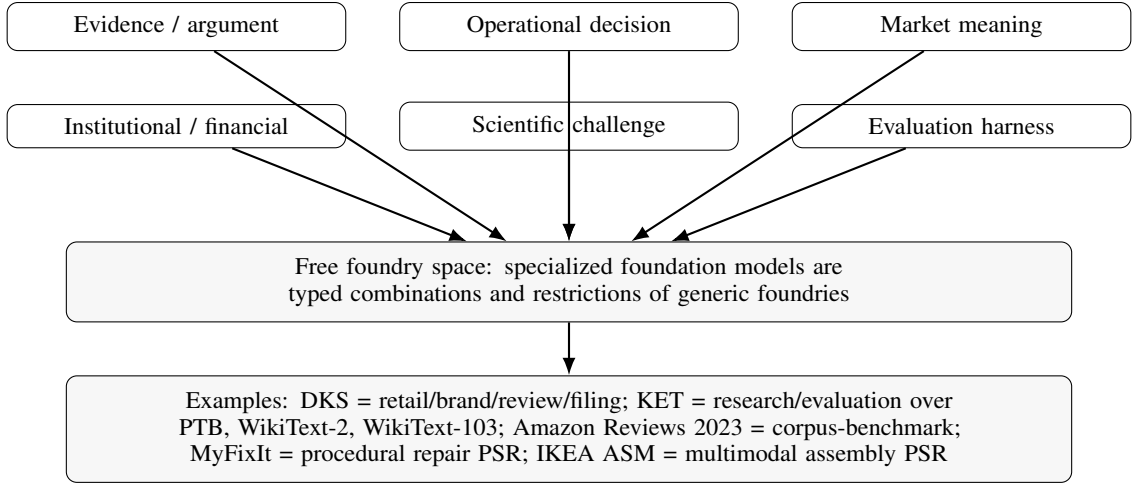

The central design principle is:
\[
  \text{representation} = \text{covering} + \text{gluing}.
\]
Documents are not reduced to one vector; they become overlapping regions with
local logics.  Brands are not reduced to personas; they become overlapping
customer, product, promise, and channel contexts.  Companies are not reduced to
summaries; they become overlapping financial, risk, narrative, and market
regions.  Agents are not reduced to policy vectors; they become overlapping
state, action, goal, tool, and evaluation contexts.  The system is engineered
so that agreement and non-agreement are both durable artifacts.

\textsc{Odyssey} is implemented as a five-agent stack.  \textsc{Scylla} translates a
human request into a model-design brief and explains what the resulting foundry
can responsibly answer.  \textsc{Homer} turns the brief into an executable
workflow skeleton.  \textsc{Athena} assigns the representational semantics:
covers, local models, truth values, restrictions, gluing rules, obstruction
policies, and cross-sheaf bridges.  \textsc{Prometheus} is the engine-room
operator: it instantiates the plan as a Topos World Model, evaluates local
predicates, audits gluing, emits obstruction reports, and builds Scylla-facing
dashboards and durable JSON artifacts.  \textsc{Toulmin} is the argumentation
agent: it turns maintained foundry state into warranted claims, explicit
grounds, backing, qualifiers, rebuttals, and context-sensitive justification.

\paragraph{Contributions.}
This paper makes six contributions.
\begin{enumerate}[leftmargin=*]
  \item We introduce \textsc{Odyssey}, a foundry architecture for building
  inspectable foundation models as sheaf-like families of local predictive and
  logical models rather than as single opaque embeddings.
  \item We specify the implementation contract among
  \textsc{Scylla}, \textsc{Homer}, \textsc{Athena}, \textsc{Prometheus}, and
  \textsc{Toulmin},
  including the durable artifacts exchanged at each boundary.
  \item We describe a foundry algebra in which generic foundries can be
  composed, specialized, restricted, glued, lifted, and audited to produce
  domain-bound foundry instances.
  \item We introduce Foundry SQL (FSQL), a small typed query surface for
  slicing maintained foundry artifacts, and TICKET, Topos Integration using
  Causal Kan Extension Transformers, for admitting external or pre-built
  Prometheus models into durable \textsc{Odyssey} state.
  \item We formalize the local predictive-state and sheaf-theoretic core:
  contexts, covers, local sections, finite truth values, restrictions, gluing
  diagnostics, obstruction records, and promotion gates.
  \item We document the current foundry families and examples implemented in
  the repository, including storefront, brand, corporation, Dick's Sporting
  Goods, Amazon Reviews 2023, MyFixIt, Indus Script, TCC 44K, research-program,
  assistant-build, embedding-evaluation, and grounded Toulmin/local-LLM
  comparison foundries.
\end{enumerate}

\section{Related Work}
\label{sec:related}

\textsc{Odyssey} is primarily a proposal about \emph{foundation-model construction}: how
a durable domain model is specified, assembled, inspected, repaired,
transported, and maintained after the initial sources have been collected.
This places the system at the
intersection of foundation-model engineering, data-centric model development,
domain-specific foundation models, agentic orchestration, causal and
argumentative NLP, and sheaf-like consistency management.

The central departure is the unit of construction.  Much foundation-model work
describes a model as a checkpoint, service, dataset, benchmark score, or
application pipeline.  \textsc{Odyssey} instead treats a model-building effort as a
\emph{foundry}: a typed artifact family with a source surface, local contexts,
representation families, restriction maps, gluing rules, obstruction ledgers,
promotion gates, refresh obligations, and human-facing views.  The five \textsc{Odyssey}
agents decompose this construction problem.  \textsc{Scylla} fixes the
human-facing contract, \textsc{Homer} makes the workflow executable,
\textsc{Athena} supplies the representation and gluing laws,
\textsc{Prometheus} materializes the local world-model artifacts, and
\textsc{Toulmin} turns maintained state into warranted, qualified, rebuttable
claims.  Related work is therefore best read through this construction stack:
which parts of the lifecycle does each prior approach make explicit, and which
parts remain implicit in model state, prompt context, or ad-hoc application
code?

\paragraph{Construction of foundation models.}
The term \emph{foundation model} was introduced to name the emerging class of
models trained on broad data at scale and adapted across many downstream tasks
\citep{bommasani2021foundationModels}.  Much of the subsequent literature
treats construction as an engineering pipeline: assemble data, train at scale,
serve efficiently, adapt to downstream tasks, and evaluate risks.  Surveys of
training and serving systems emphasize the computational, memory, bandwidth,
parallelism, and deployment constraints that make foundation-model development
a systems problem as much as a modeling problem
\citep{zhou2024trainingServingFoundationModels}.  Data-centric work argues
that model construction also depends on curation, attribution, benchmark
design, knowledge transfer, and inference-time context, all of which are
underrepresented if foundation models are described only by architecture and
parameter count \citep{xu2024dataCentricLLMs}.

Recent work on domain-specific foundation models makes the construction
problem more explicit: the goal is not merely to fine-tune a general model, but
to customize data, objectives, architecture, adaptation strategy, and evaluation
to the structure of a particular industry or scientific domain
\citep{chen2024domainSpecificFoundationModels}.  \textsc{Odyssey} agrees with this
domain-specific turn but shifts the unit of construction.  A foundry is not
only a trained model checkpoint.  It is a typed recipe for building a durable
domain model: source surfaces, local contexts, representation families,
restriction maps, gluing rules, obstruction policies, update contracts, and
human-facing views.

There is also a growing engineering literature around how foundation models are
released, observed, guided, and embedded into agentic systems.  The Model
Openness Framework argues that reproducibility and usability require releasing
components across the model development lifecycle, not only weights
\citep{white2024modelOpennessFramework}.  Work on foundation-model ``sherpas''
studies ways that agents can guide models through knowledge augmentation,
prompting, reasoning support, updating, and output evaluation
\citep{bhattacharjya2024foundationModelSherpas}.  Agent design-pattern
catalogues similarly treat foundation-model applications as architectures with
memory, planning, tool use, reflection, and accountability tradeoffs
\citep{liu2024agentDesignPatternCatalogue}.  \textsc{Odyssey} is closest to this
architectural view, but its emphasis is representational and argumentative
rather than only procedural: Scylla, Homer, Athena, Prometheus, and Toulmin
make the domain model itself an inspectable sheaf-like artifact before it is
used by a human or agent.

\paragraph{Causal relation extraction from text.}
There is a long line of work on identifying causal relations in natural
language, from cue-phrase and pattern-based systems to neural classifiers; see
surveys of causal relation extraction and event causality identification
\citep{yang2022causalSurvey,cheng2025eventCausalitySurvey}.  Classical systems
usually predict whether a pair of spans or events in a sentence stands in a
causal relation, while corpus-scale work connects such local predictions to
event forecasting or explanatory retrieval \citep{radinsky2012learningCausality}.
\textsc{Prometheus} uses such extracted relations as evidence units, but the
paper's object of study is downstream: how thousands of local claims should be
localized, compared, transported, or blocked across corpus regions.

\paragraph{Causal knowledge bases and graphs from corpora.}
Causal knowledge-base projects mine cause--effect tuples from large corpora and
aggregate them into graph-structured resources
\citep{hassanzadeh2020causalKB}.  This graph-building perspective is close to
the first \textsc{Democritus} contribution, where LLM-generated causal
statements are compiled into local causal models and larger causal atlases
\citep{mahadevan2025largecausalmodelslarge}.  \textsc{Prometheus} extends that line by
treating local graphs and cSQL rows as observations for local causal PSRs.  The
global object is therefore not one merged graph but a sheaf-like family of
charts whose overlaps reveal agreement, drift, contradiction, and
underdetermination.

\paragraph{LLMs for causal discovery and reasoning.}
A growing literature asks whether LLMs can propose causal directions, graph
structures, interventions, or counterfactual explanations from variable
descriptions and textual context
\citep{kiciman2024causalLLM,le2024multiagentCausalDiscovery}.
\textsc{Prometheus} is deliberately more conservative.  It does not treat the
LLM as an oracle for ground-truth causal discovery.  Instead, the LLM helps
surface causal discourse: claims, mechanisms, modifiers, regimes, and source
passages that can be normalized, audited, and compared.  Local intervention
probes in the atlas are therefore model-internal research tests unless paired
with external data and identification assumptions.

\paragraph{Agentic systems for automated scientific discovery.}
Recent systems also aim to automate larger portions of the scientific workflow.
The AI Scientist-v2, for example, uses agentic tree search to propose
hypotheses, design and execute machine-learning experiments, analyze and
visualize results, and write scientific manuscripts
\citep{yamada2025aiScientistV2}.  This line of work is close in ambition to
\textsc{Prometheus}: both ask how AI systems can participate in scientific
discovery rather than merely answer questions about existing papers.  The
emphasis is different.  AI Scientist-v2 organizes autonomous experimentation
and manuscript generation, primarily in machine-learning research settings.
\textsc{Prometheus} instead constructs an explicit causal topos world model
from heterogeneous research artifacts---text, data, figures, source code, and
scientific models---so that local claims, gluing failures, evidentiary limits,
and grounded counterfactual revisions remain inspectable.  In this sense,
\textsc{Prometheus} can be viewed as a complementary world-model layer for
scientific agents: it records what a research substrate supports, where it
does not glue, and which counterfactuals can actually be evaluated.

\paragraph{Prometheus as prior work.}
The Prometheus paper introduced Topos World Models for deep causal research:
language-derived causal episodes, local causal PSRs, restriction maps, gluing
diagnostics, persistent causal atlases, and grounded counterfactual rebuilding
over scientific substrates \citep{mahadevan2026prometheus}.  We do not repeat
that contribution here.  \textsc{Odyssey} treats Prometheus as one engine-room component
inside a broader foundation-model foundry stack.  The new questions in this
paper are architectural and empirical: how human requests become typed foundry
expressions; how Scylla, Homer, Athena, Prometheus, and Toulmin exchange
artifacts; what families of foundries have been built; and whether the first
procedural PSR foundry, MyFixIt, yields useful preliminary retrieval behavior.

\paragraph{Causality-aware NLP.}
More broadly, causal ideas have been used to study text effects,
counterfactual augmentation, representation robustness, and explanations for
NLP systems \citep{feder2022causalNLP}.  \textsc{Prometheus} points in the
opposite direction: it uses NLP and LLM extraction to construct explicit causal
artifacts for human research.  The Claims Atlas is meant to be inspected,
corrected, extended, and rerun, so provenance and gluing failures are part of
the output rather than post-hoc debugging aids.

\section{\textsc{Odyssey} System Architecture}
\label{sec:odyssey-architecture}

\textsc{Odyssey} separates foundation-model construction into five roles with explicit
artifact boundaries.  The point of this separation is not only software
modularity.  It is also epistemic discipline: a human request, an executable
workflow, a representational semantics, an instantiated world model, and an
argumentative justification layer should not be confused with one another.

\begin{table}[t]
\centering
\small
\begin{tabular}{p{0.18\linewidth}p{0.33\linewidth}p{0.38\linewidth}}
\toprule
Module & Responsibility & Current artifact boundary \\
\midrule
\textsc{Scylla} & Human-facing model brief and answer contract. &
\texttt{scylla\_model\_brief.json}: user goal, entities, decisions,
uncertainties, trust questions, and expected outputs. \\
\textsc{Homer} & Workflow orchestration plan. &
\texttt{model.homer.json}: ordered execution steps, refresh obligations,
tool/API surfaces, and required audits. \\
\textsc{Athena} & Sheaf typing and representation semantics. &
\texttt{athena\_sheaf\_plan.json}: cover, local representation types, finite
truth values, overlaps, gluing rule, source sheaves, and cross-sheaf bridges. \\
\textsc{Prometheus} & World-model instantiation and audit. &
\texttt{prometheus\_world\_model.json}, \texttt{gluing\_audit.json}, HTML
explorers, Scylla dashboards, ingestion packets, and callback bundles. \\
\textsc{Toulmin} & Argument construction, critique, and rebuttal. &
\texttt{argument\_ticket.json}: claims, grounds, warrants, backing,
qualifiers, rebuttals, applicable contexts, and obstruction checks. \\
\bottomrule
\end{tabular}
\caption{The implemented \textsc{Odyssey} handoff contract.  Each module emits a durable
object that can be inspected independently and reused by later stages.}
\label{tab:odyssey-modules}
\end{table}

The current implementation is deliberately deterministic in several places.
For example, Scylla uses domain templates to compile broad requests into model
briefs, and Prometheus validation runs evaluate finite truth predicates using
stable rules.  This is a feature of the present design stage: it lets us test
the foundry contract, artifact schemas, and user-facing surfaces before
replacing individual components with richer parsers, learned classifiers,
retrieval systems, or external model calls.

\paragraph{Scylla: request to brief.}
Scylla names the user's request in human terms.  A brief records the domain,
goal, entities, decisions, uncertainties, trust questions, and output contract.
For a corporation foundry, for example, Scylla identifies financial statements,
risk factors, management narrative, strategy signals, and market context as
objects of concern.  For an Indus Script foundry, it separates symbols,
inscription sequences, statistical structure, decipherment hypotheses,
archaeological context, and uncertainty.  This prevents the system from
answering as if every domain were a generic retrieval task.

\paragraph{Homer: brief to workflow.}
Homer compiles the brief into an executable program skeleton: collect sources,
request an Athena plan, instantiate the Prometheus model, run gluing audits,
emit dashboards, and schedule refresh obligations.  In process foundries such
as research-program, assistant-build, or evaluation-harness foundries, Homer
also records maintenance cadence, policy gates, simulation surfaces, and
artifact rebuild obligations.

\paragraph{Athena: workflow to sheaf plan.}
Athena owns the representational semantics.  It chooses the local contexts that
form the cover, assigns each context a local representation type, lists
predicates that can be evaluated inside the context, declares overlaps, and
defines gluing logic.  The current truth-value strategy uses the finite
partition
\[
  \bot,\ \mathrm{WEAK},\ \mathrm{PLAUSIBLE},\ \mathrm{SUPPORTED},\ \top,
\]
translated for users as unsupported, weak signal, plausible, well supported,
and directly confirmed.  An overlap glues when shared predicates are at least
plausible and no paired predicate is bottom; otherwise Athena asks Prometheus
to emit an obstruction record and a next-observation recommendation.

\paragraph{Prometheus: plan to Topos World Model.}
Prometheus instantiates Athena's plan as a compact Topos World Model.  It emits
local sections, restriction maps, gluing results, integration layers, update
status, evidence sources, and HTML views.  For domains with source sheaves,
Prometheus also builds auxiliary sheaves such as the Dick's Sporting Goods
shopping-experience sheaf, the Amazon Reviews 2023 corpus sheaf, the Indus
evidence sheaf, and procedural PSR candidates over MyFixIt repair manuals.
Prometheus is therefore the engine room of \textsc{Odyssey}, but not the entire system:
it runs after Scylla, Homer, and Athena have fixed the intent, workflow, and
representation contract.

\paragraph{Toulmin: foundry state to argument.}
Toulmin is the argumentation agent layered over maintained foundry state.  It
does not replace Prometheus; Prometheus remembers and audits local sections,
while Toulmin turns those sections into defensible arguments.  A Toulmin ticket
records a claim, its data or grounds, the warrant licensing the inference,
backing for that warrant, qualifiers describing the force of the claim, rebuttals
that restrict where it applies, and the contexts over which the argument can be
transported.  This makes the evidence/argument foundry an active structure:
\textsc{Odyssey} can answer not only ``what does the foundry contain?'' but also ``why is
this claim warranted here, and where would the argument fail?''

\section{Foundry Algebra}
\label{sec:foundry-algebra}

\textsc{Odyssey} is organized around a small algebra of foundries.  A generic foundry is
a reusable representational or process architecture.  Representational
foundries include Toulmin-style evidence/argument, operational decision,
institutional/financial, market meaning, scientific challenge, document
coherence, database-atlas incorporation, and Prometheus-model incorporation
foundries.  Process foundries include research-program, assistant-build,
evaluation-harness, product-development, codebase-evolution, and
result-communication foundries.

A specialized foundry is a typed expression built from generic foundries.  In
the current demos:
\[
\begin{array}{l}
\mathrm{DKS} =
\mathrm{specialize}(
\mathrm{compose}(\mathrm{storefront\_operations},
\mathrm{brand\_meaning},\\
\qquad
\mathrm{corporation\_financial},
\mathrm{review\_evidence}),
\mathrm{domain}=\text{``Dick's Sporting Goods''}),\\[2mm]
\mathrm{Indus} =
\mathrm{specialize}(
\mathrm{compose}(\mathrm{scientific\_challenge},
\mathrm{evidence\_argument},\\
\qquad
\mathrm{symbol\_sequence\_model},
\mathrm{hypothesis\_lattice},
\mathrm{archaeology\_context}),
\mathrm{domain}=\text{``Indus Script''}).
\end{array}
\]
The formal operations include composition, specialization, restriction,
gluing, product, coproduct, quotient, and lift.  In implementation terms, these
operations determine which local sections must exist, which overlaps must be
checked, which source sheaves may be bridged, and which artifacts become
durable state.

This algebra gives \textsc{Odyssey} a discipline for growth.  New demos should not enter
the system as unrelated scripts.  They should declare the generic foundries
they instantiate, the target domain, the local contexts, the restriction maps,
and the promotion gate that decides when an output can become maintained
foundry state.

\section{Foundry SQL and TICKET Admission}
\label{sec:fsql-ticket}

\textsc{Odyssey} uses Foundry SQL (FSQL) as a compact typed query surface over foundry
artifacts.  Standard SQL queries build typed constructions over tables; FSQL
generalizes the same idea from rows and joins to foundry sections, source
sheaves, restriction maps, gluing checks, and promotion gates.  In this
analogy, tables become local charts and artifact stores, rows become claims or
model states, foreign keys become shared identifiers and provenance links,
joins become gluing checks, views become Scylla-facing dashboards, and
constraints become Athena and Prometheus audits.

\subsection{Universal Foundry Learners}
\label{sec:ufl}

The categorical template behind the admission story is the Universal Decision
Learner (UDL) pattern from \citet{mahadevanCatAGIBook}.  UDL characterizes a
decision rule by its universal property rather than by one numerical algorithm.

\begin{definition}[Universal Decision Learner]
\label{def:udl}
Given local decision data \(J:\mathcal{O}\to\mathcal{C}\) and
\(D:\mathcal{O}\to\mathcal{D}\), a Universal Decision Learner is the composite
Kan-extension semantics
\[
  \mathrm{UDL}_{J}(D)
  =
  \mathrm{Ran}_{J}(\mathrm{Lan}_{J}D),
\]
whenever the displayed Kan extensions exist, with the evident restriction or
precomposition steps understood when needed for typing.  More generally, UDL
denotes any decision semantics obtained by composing left and right Kan
extensions along problem-specific inclusions of local context into global
context.
\end{definition}

Operationally, the construction has two stages:
\[
  \text{local decision data}
  \xrightarrow{\ \mathrm{Lan}\ }
  \text{rolled-out candidates}
  \xrightarrow{\ \mathrm{Ran}\ }
  \text{globally consistent decisions}.
\]
The left Kan stage expands, aggregates, or propagates information.  The right
Kan stage filters, glues, or enforces consistency.  In reward-enriched settings,
for example over the max-plus semiring, the left Kan stage recovers the familiar
dynamic-programming rollout
\[
  (\mathrm{Lan}_{J}D)(c)
  =
  \max_{Jd\to c}\bigl(D(d)+w(d\to c)\bigr),
\]
while the right Kan stage imposes the tightest downstream inequalities induced
by all continuations.  In deterministic one-step planning this specializes to
the Bellman recurrence
\[
  V(s)=\max_a\{r(s,a)+V(T(s,a))\}.
\]

\begin{definition}[Universal Foundry Learner]
\label{def:ufl}
Let \(x\) be a candidate source artifact, \(Y\) a target foundry, and
\[
  F_{x,Y}:\C_x\to\C_Y
\]
the declared source-to-target interface functor.  If
\(X:\C_x^{op}\to\mathsf{Data}\) is the candidate presheaf, define its admitted
target-side rollout by
\[
  A=\mathrm{Lan}_{F_{x,Y}}X.
\]
The Universal Foundry Learner associated with this TICKET interface is the
right-Kan consistency envelope
\[
  \mathrm{UFL}_{x,Y}(X)
  =
  \mathrm{Ran}_{F_{x,Y}}F_{x,Y}^{*}A.
\]
Thus UFL is the foundry-construction specialization of UDL: left Kan admission
constructs the least target-side foundry candidate, and right Kan consistency
collects the target obligations visible through the same interface.
\end{definition}

\begin{theorem}[Canonical foundry rollout]
\label{thm:ufl-rollout}
Let \(X:\C_x^{op}\to\mathsf{Data}\) and \(F_{x,Y}:\C_x\to\C_Y\) be as above.
If \(A=\mathrm{Lan}_{F_{x,Y}}X\) exists, then for every target foundry model
\(G:\C_Y^{op}\to\mathsf{Data}\) there is a natural bijection
\[
  \mathrm{Nat}(A,G)
  \cong
  \mathrm{Nat}(X,F_{x,Y}^{*}G).
\]
Hence \(A\) is initial among target-side foundry candidates receiving the local
source artifact \(X\).
\end{theorem}

\begin{proof}
This is the defining universal property of the left Kan extension, read in the
foundry category.  To compare the rolled-out target candidate \(A\) with any
target model \(G\) is equivalently to compare the source artifact \(X\) with
the restriction of \(G\) along the declared TICKET interface.
\end{proof}

\begin{theorem}[Canonical foundry consistency]
\label{thm:ufl-consistency}
If \(\mathrm{Ran}_{F_{x,Y}}F_{x,Y}^{*}A\) exists, then every target model
\(G:\C_Y^{op}\to\mathsf{Data}\) whose restriction is compatible with the
rolled-out candidate admits a canonical comparison map
\[
  G\longrightarrow \mathrm{UFL}_{x,Y}(X).
\]
When this comparison is an isomorphism, \(G\) computes the same foundry
semantics as the Universal Foundry Learner.
\end{theorem}

\begin{proof}
This is the defining universal property of the right Kan extension.  It says
that globally defined foundry models satisfying the source-visible target
obligations factor through the canonical consistency envelope.
\end{proof}

The first maintained FSQL admission construction is TICKET:
\[
  \textsc{TICKET}
  = \text{Topos Integration using Causal Kan Extension Transformers}.
\]
The phrase ``Kan Extension Transformer'' comes from the broader categorical
learning program developed in \citet{mahadevanCatAGIBook} and in the
Kan Extension Transformers arXiv manuscript
\citep{mahadevan2026kanExtensionTransformers}.  In the language-modeling
setting, a KET is a neural sequence model whose inductive bias is organized by
Kan-extension-style transport between local contexts.  TICKET uses the same
categorical idea, but not the same object.  It is not a text generator and it is
not a replacement for the KET language model.  It is an admission transformer:
a typed procedure that extends a source model along a declared map into an
\textsc{Odyssey} target foundry, then tests whether the transported state satisfies the
target cover, restrictions, and gluing laws.  \Cref{app:ticket-specs} gives the
appendix-level categorical specification, including the source-to-target
functor, the left-Kan admission move, the right-Kan UDL consistency pass, and
the formal gluing operation used by the GUI TICKET cards.

Thus, TICKET is the bridge between \emph{model construction} and
\emph{maintained foundry state}.  A Prometheus run, CSQL atlas, benchmark
repository, repair-manual slice, or scientific-evidence bundle may already have
local structure, but \textsc{Odyssey} does not admit it merely because it exists.  The
source must be sliced into local sections, transported into the target foundry,
checked against Athena's representation contract, and compared with maintained
state.  Only then can it be promoted, quarantined for later repair, or preserved
as an obstruction record.

\begin{figure}[t]
\centering
\scriptsize
\begin{tikzpicture}[
  node distance=6mm and 8mm,
  box/.style={draw,rounded corners,align=center,text width=0.24\linewidth,inner sep=5pt},
  wide/.style={draw,rounded corners,align=center,text width=0.34\linewidth,inner sep=5pt,fill=black!3},
  gate/.style={draw,rounded corners,align=center,text width=0.32\linewidth,inner sep=5pt,fill=blue!4},
  result/.style={draw,rounded corners,align=center,text width=0.25\linewidth,inner sep=5pt,fill=green!4},
  arrow/.style={-Latex,thick},
  audit/.style={-Latex,thick,dashed}
]
\node[box] (source) {Source artifact \(X:\C_x^{op}\to\mathsf{Data}\)\\
documents, events, claims, traces};
\node[wide,right=of source] (lan) {Left-Kan rollout\\
\(\mathrm{Lan}_{F_{x,Y}}X=A\)\\
document-claim candidate};
\node[box,right=of lan] (target) {Target foundry \(Y\)\\
site \(\C_Y\), cover, restrictions, Toulmin roles};
\node[gate,below=of lan] (ran) {Right-Kan / pullback scrutiny\\[1mm]
\(\displaystyle \mathrm{Ran}_{F_{x,Y}}(F_{x,Y}^{*}A)\)\\[1mm]
consistency of claim, grounds, warrant};
\node[box,below=of source] (maintained) {Maintained state \(M_Y\)\\
existing local sections and obstruction ledger};
\node[box,below=of target] (toulmin) {\textsc{Toulmin} ticket\\
claim, grounds, warrant, qualifier, rebuttal};
\node[result,below=of ran] (decision) {TICKET decision\\
promote, quarantine, or preserve obstruction};
\draw[arrow] (source) -- node[above] {\(F_{x,Y}\)} (lan);
\draw[arrow] (lan) -- (target);
\draw[audit] (lan) -- node[right] {scrutinize rollout} (ran);
\draw[audit] (ran) -- (toulmin);
\draw[arrow] (maintained) -- node[below left] {glue / compare} (decision);
\draw[arrow] (target) -- node[right] {argument roles} (toulmin);
\draw[arrow] (toulmin) -- (decision);
\draw[arrow] (lan.south east) .. controls +(9mm,-9mm) and +(11mm,9mm) ..
  node[right,pos=0.55] {candidate claim} (decision.east);
\end{tikzpicture}
\caption{TICKET as \textsc{Odyssey}'s admission and consistency operator.  The left Kan
extension rolls a structured source artifact into the target foundry as a
document-claim candidate.  The right Kan / pullback pass supplies the UDL-style
consistency check: the same interface asks whether the candidate claim, its
grounds, and its warrant survive Toulmin scrutiny against the target cover and
maintained foundry state.  Promotion occurs only when the rolled-out claim is
warranted, qualified, and compatible with the existing obstruction ledger.}
\label{fig:ticket-lan-ran}
\end{figure}
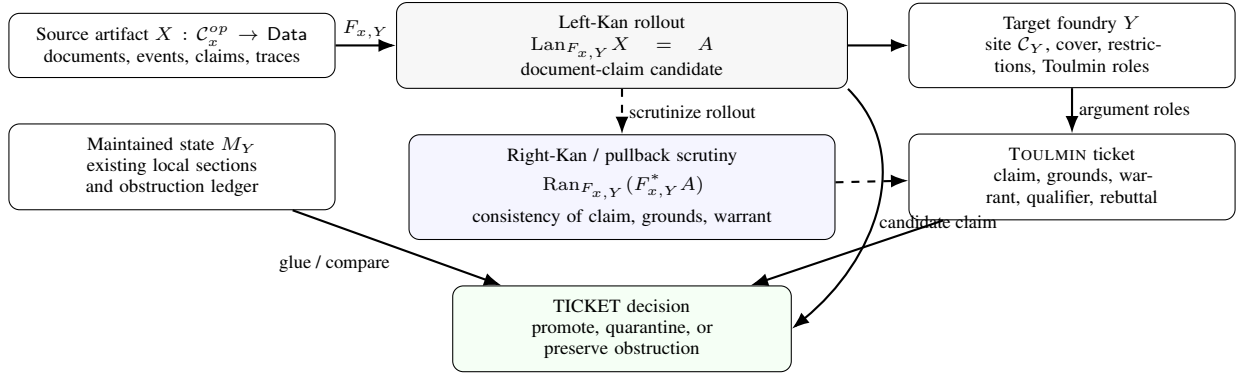

\begin{table}[t]
\centering
\small
\begin{tabular}{p{0.20\linewidth}p{0.34\linewidth}p{0.34\linewidth}}
\toprule
Construction & Source and target & Output \\
\midrule
Language KET & Text histories, tests, and corpus-local contexts transported
inside a language-modeling problem. & Predictive sequence model and
corpus-local evaluation claims. \\
TICKET & External foundry artifacts, Prometheus runs, CSQL atlases, benchmark
repositories, and maintained \textsc{Odyssey} target foundries. & Admission packet:
typed map, restriction checks, gluing audit, promotion decision, obstruction
ledger, and refresh obligations. \\
\bottomrule
\end{tabular}
\caption{TICKET specializes the Kan Extension Transformer idea to foundry
ingestion.  The transformer acts on structured model artifacts and sheaf
contracts, not directly on next-token prediction.}
\label{tab:ticket-ket}
\end{table}

Operationally, TICKET admits a pre-built or external model into \textsc{Odyssey} by
slicing the source run, transporting its cover into a target foundry by a
Kan-extension-style transformer, checking causal orientation and provenance,
auditing restrictions and gluing, and then deciding whether the candidate state
is promoted, quarantined, or held as an obstruction.  The current admission
checks are:
\[
\begin{gathered}
  \texttt{typed\_manifest},\ \texttt{subobject\_classifier},\
  \texttt{restriction\_maps},\\
  \texttt{gluing\_audit},\
  \texttt{j\_closure},\ \texttt{promotion\_gate}.
\end{gathered}
\]

A typical FSQL surface has the form
\begin{verbatim}
SELECT object_type, object_id, status, source
FROM prometheus_runs
SLICE BY run("tcc_44k")
TICKET BY target_foundry
\end{verbatim}
and returns the transformer, target foundry, checks, promotion record, and
refresh obligations that govern admission.  The TCC 44K example treats the
Testing Causal Claims cSQL atlas as a source causal-claims foundation model:
nodes, edges, support rows, methods, journals, p-values, and causal-direction
metadata are lifted into a causal-claims foundry while weak joins and direction
uncertainties remain visible as guardrails.

\section{BRIDGE/SKFM Refinement and SkillOpt Integration}
\label{sec:bridge-skfm-skillopt}

The ICML tutorial version of \textsc{Odyssey} includes a new experimental lane
that connects three pieces of the system: BRIDGE/SKFM causal-geometry screens
\citep{mahadevan2026latentconfoundedcausaldiscovery}, TICKET admission, and SkillOpt skill
optimization.  The goal is not to collapse these into one optimizer.
BRIDGE/SKFM supplies a latent causal refinement foundry, \textsc{Odyssey}
supplies the local sections, Toulmin qualifications, restriction maps, and
promotion gates, and SkillOpt treats the natural-language extraction policy
itself as a trainable artifact.

\paragraph{BRIDGE/SKFM foundry contract.}
The implemented contract is a generic
\emph{BRIDGE/SKFM latent causal refinement foundry}.  It refines existing
\textsc{Odyssey} foundries by mapping local causal, argument, filing, repair,
or product-feedback charts into typed variables, influence masks, candidate
edge masks, Frobenius/Lie residual ratios, latent-obstruction ledgers, and
optional SKFM vector-field diagnostics.  Four source-family lanes are currently
declared:
\[
  \text{Democritus PDF causal corpus},\quad
  \text{MyFixIt repair PSRs},\quad
  \text{10-K filing arguments},\quad
  \text{brand/product feedback}.
\]
The promotion rule is deliberately conservative: residual-bearing pairs are
kept as obstruction candidates unless the source family, local section, Toulmin
warrant, qualifier, and missing-context explanation are preserved.  Numeric
SKFM training is enabled only when calibrated traces, interventions, probes, or
comparable field targets exist; otherwise the current stage is reported as a
proxy-field BRIDGE screen rather than as a learned causal vector field.

\paragraph{Stage-0 BRIDGE/SKFM experiments.}
The first Democritus PDF scale-up screened \(29\) article charts grouped into
\(6\) cluster sections using Kimi~2.5 served through a local \texttt{exo}
OpenAI-compatible endpoint.  It found \(1\) gluable mechanism overlap and \(7\)
overlaps requiring review.  The Lie-BRIDGE adapter emitted \(15\) language
variables, \(29\) claim fields, \(15\) candidate edges, and \(6\) latent
obstruction pairs.  These artifacts are useful precisely because they do not
pretend that latent variables have been discovered as named hidden causes.
They are reviewable residual witnesses over causal-PSR cells, suitable for
later numeric BRIDGE/SKFM scoring when stronger intervention or trajectory
data are available.

The same contract has deterministic adapters for MyFixIt, 10-K, brand feedback,
and Democritus language-extraction slices.  In the MyFixIt lane, repair-manual
action-observation fields are converted into tool, part, verb, image-context,
and future-test variables.  A follow-up retrieval experiment separates two
effects: simply adding BRIDGE vocabulary did not improve the retrieval metric,
whereas adding manual-context upper-bound fields improved MRR and moved the
Phillips-screw failure slice to rank \(1\).  The 10-K, brand-feedback, and
Democritus language-extraction adapters likewise improve targeted section,
theme, stance, or causal-language retrieval while retaining latent obstruction
rows for audit.

\paragraph{SkillOpt as policy optimization for TICKET admission.}
SkillOpt is integrated through an
\texttt{odyssey\_causal\_claim} environment
\citep{yang2026skilloptexecutivestrategyselfevolving}.  Each rollout gives a model a
short scientific, news, or business passage and asks for a strict JSON causal
claim bundle: claim, cause, effect, relation, qualifier, evidence span,
mechanism, scope, guardrail, and TICKET status.  The scorer combines field
accuracy with \textsc{Odyssey}-specific diagnostics: TICKET compatibility,
diagrammatic consistency energy inspired by infinitesimal-causality tests
\citep{mahadevan2026infinitesimalcausality}, graph similarity, and a KET-style
transfer score across local contexts.  The optimized object is the Markdown skill that
instructs the target model how to avoid causal overclaiming, preserve source
scope, downgrade associations, and emit admission-compatible fields.

\begin{table}[t]
\centering
\small
\begin{tabular}{p{0.22\linewidth}p{0.20\linewidth}p{0.20\linewidth}p{0.26\linewidth}}
\toprule
Run & Selection hard & Held-out hard / soft & Admission effect \\
\midrule
Seed 42 & \(0.9847\) & \(1.0000 / 0.9881\) &
Baseline held-out hard score \(0.0\); optimized skill reaches full hard
success on the held-out split. \\
Seed 43 & \(0.9847\) & \(1.0000 / 0.9881\) &
Independent seed reproduces the full held-out hard success. \\
Seed 44 & \(0.9847\) & \(0.8333 / 0.8333\) &
Optimization still improves over a \(0.0\) hard baseline but leaves one
held-out failure for audit. \\
\bottomrule
\end{tabular}
\caption{SkillOpt runs for the \texttt{odyssey\_causal\_claim} environment.
All three runs use the same split and local OpenAI-compatible model route.  The
important artifact is not only the final score, but the optimized skill and the
failure traces that explain which admission rules, qualifiers, or guardrails
changed.}
\label{tab:skillopt-odyssey-causal-claim}
\end{table}

This integration gives a concrete empirical role to the categorical machinery.
Diagrammatic backpropagation contributes consistency losses over the claim
diagram; infinitesimal-causality diagnostics test local causal variation;
geometric-transformer-style comparison scores relational fit between predicted
and admitted claim graphs; KET-style transfer tests whether a skill learned in
one local regime extends to neighboring domains; and BRIDGE/SKFM records
residual-bearing pairs that should remain local until stronger evidence
justifies promotion.  SkillOpt then optimizes the human-readable policy used by
the model, while \textsc{Odyssey} preserves the admission ledger that decides
whether the resulting claims become maintained foundry state.

\section{From \textsc{Democritus} to \textsc{Prometheus}}

\textsc{Democritus} is the predecessor pipeline: a language-to-causal-model
system for compiling documents into local causal models, causal databases, and
interactive diagnostic artifacts
\citep{mahadevan2025largecausalmodelslarge}.  A public implementation of the
released \textsc{Democritus} client is available as
\texttt{Democritus\_OpenAI} \citep{mahadevanDemocritusOpenAI}.  The broader categorical machine
learning background for this line of work is developed in
\citet{mahadevanCatAGIBook}.  It extracts
local causal claims, organizes them into causal triples or local causal models,
and stores structured outputs in cSQL-like causal tables.  This is already
useful: it turns unstructured text into queryable causal objects.

\textsc{Prometheus} changes the central object.  Instead of treating a local DAG or
causal table as the final representation, \textsc{Prometheus} treats it as evidence for a
local predictive-state model.  A local model records not only that \(X\) causes
or influences \(Y\), but which histories and tests are present, what the model
predicts under those tests, how much support each cell has, and where the
evidence came from.  The global object is not one merged DAG.  It is a sheaf-like
family of local models, equipped with restriction and gluing diagnostics.

\begin{table}[t]
\centering
\begin{tabular}{lll}
\toprule
System & Local object & Global object \\
\midrule
\textsc{Democritus} & Causal claim, DAG, cSQL row & Causal synthesis over extracted tables \\
\textsc{Prometheus} & Causal PSR over a context & Topos World Model / causal atlas \\
\bottomrule
\end{tabular}
\caption{\textsc{Prometheus} inherits the extraction discipline of \textsc{Democritus} but shifts
the representation from graph-centered synthesis to predictive-state sheaves.}
\label{tab:democritus-prometheus}
\end{table}

This shift matters because deep research is rarely about finding one graph.  It
is about understanding which local graphs, claims, and predictions are
transportable.  Topological organization gives the system a way to say: these
regions agree, these overlap but pull apart, and this claim should not be moved
without additional evidence.

\section{\textsc{Toulmin}: Argumentation as a Foundry Agent}
\label{sec:toulmin-agent}

The Democritus--Prometheus lineage made causal extraction durable: a document
could yield local causal triples, local causal models, and eventually local PSR
sections.  What it did not yet make first class was the formal structure of an
argument.  Recent NLP work has shown that Toulmin's theory can also be used as
a zero-shot prompt frame for explicating informal arguments into claims,
reasons, and warrants \citep{gupta2024harnessing}.  In \textsc{Odyssey},
\textsc{Toulmin} fills the complementary systems gap.  Following Toulmin's
analysis of practical reasoning as field-dependent, defeasible, and
context-sensitive \citep{toulmin1958uses}, the new agent treats argumentation as
a primary foundry capability rather than as prose wrapped around a causal
triple.

\begin{table}[t]
\centering
\small
\begin{tabular}{ll}
\toprule
Toulmin element & \textsc{Odyssey} / sheaf reading \\
\midrule
Claim & Local proposition or predictive-state assertion \\
Data / grounds & Source observations, extracted claims, traces, or measurements \\
Warrant & Local rule, morphism, or transport principle licensing the claim \\
Backing & Neighboring sections, source sheaves, precedents, models, or audits \\
Qualifier & Finite truth value, confidence, scope, or force of assertion \\
Rebuttal & Obstruction, conflicting section, missing warrant, or failed restriction \\
Argument field & Local site, domain context, jurisdiction, regime, or cover element \\
\bottomrule
\end{tabular}
\caption{Toulmin structure gives the evidence/argument foundry an explicit
fiber of argumentative roles.  Gluing then tests not only whether claims agree,
but whether their warrants, backing, qualifiers, and rebuttals remain compatible
on overlaps.}
\label{tab:toulmin-odyssey}
\end{table}

This changes the topology of an \textsc{Odyssey} foundry.  A causal extraction such as
\(A \to B\) is no longer only a directed edge or table row.  Toulmin expands it
into a structured argument object:
\[
  (\mathrm{grounds},\mathrm{warrant},\mathrm{backing},
  \mathrm{qualifier},\mathrm{rebuttals}) \Rightarrow \mathrm{claim}.
\]
Prometheus can still build the persistent local sections, but Toulmin asks what
licenses each section as an assertion.  On an overlap, two papers may disagree
because their claims conflict, because they share a claim but use incompatible
warrants, because one paper supplies a rebuttal that restricts the other, or
because a qualifier narrows the open set over which the argument is valid.
These are different obstruction types, and \textsc{Odyssey} should preserve the
difference.

The sheaf interpretation is direct.  A Toulmin foundry is a presheaf of
argument objects over a context category: each local site carries claims,
grounds, warrants, backing, qualifiers, and rebuttals; restriction maps move
arguments to smaller regimes; gluing tests descent not merely for extracted
facts but for warranted arguments.  Thus Prometheus remembers persistent
foundry state, while Toulmin reasons over it.  The resulting answer contract is
not ``the model says \(X\),'' but ``\(X\) is warranted in these contexts, by
these grounds and warrants, with these qualifiers, unless these rebuttals or
obstructions apply.''

This is especially important for domains such as legal reasoning, policy
analysis, scientific debate, compliance, peer review, strategic planning, and
software-architecture justification.  In these settings, global Boolean
consistency is often the wrong target.  What matters is whether local arguments
cohere under the warrants and standards of their field, whether rebuttals are
visible, and whether cross-context transport is licensed.  Toulmin therefore
turns the evidence/argument foundry from a passive storage surface into an
active argumentative structure inside \textsc{Odyssey}'s sheaf design.

\paragraph{TRACE as an LLM-centric Toulmin baseline.}
Recent work has also begun to import Toulmin-style argumentation into the
evaluation of large-language-model reasoning.  TRACE
\citep{kim2026trace} is a reference-free framework for evaluating chain-of-
thought outputs by segmenting a model's reasoning into sentences, labeling
each sentence with constructive elements such as Claim, Data/Evidence,
Warrant, Backing, Qualifier, Rebuttal, Monitoring, and Evaluation, and then
combining State Validity with Transition Coherence into an interpretable
reasoning score.  This is an important point of contact for \textsc{Odyssey}: it gives
a concrete Toulmin-inspired baseline against which ToulminSheaf argument
foundries can be tested on LLM reasoning traces.

The difference is the unit of construction.  TRACE treats Toulmin structure as
an evaluator for a generated chain-of-thought sequence: the input is an LLM
reasoning trace, the output is a sentence-label sequence and scalar score.  A
ToulminSheaf treats argumentation as part of the foundry itself.  Its sites may
be paragraphs, documents, studies, filings, tables, causal triples, code
changes, tool traces, or model-generated rationales; its restriction maps move
arguments across contexts; and its gluing tests whether claims, grounds,
warrants, backing, qualifiers, and rebuttals remain compatible on overlaps.
Thus TRACE is naturally an evaluation-harness specialization inside \textsc{Odyssey},
whereas ToulminSheaf is a generic foundry agent.  Document-coherence foundries
and Democritus-style causal foundries are special cases of ToulminSheaf
foundries, while TRACE supplies an LLM-centric baseline for testing whether
ToulminSheaf gluing quality agrees with or improves upon sentence-level CoT
validity and transition-coherence scores.

\paragraph{Current grounded local-LLM implementation.}
The repository now contains a first executable Toulmin layer over maintained
\textsc{Odyssey}/Prometheus state.  The command-line interface inventories locally
cached Hugging Face and MLX models, exposes domain aliases, builds visible
Toulmin prompts, and runs a selected local model over document-grounded claim
packets.  Each packet contains an extracted Prometheus event, its source
context, the normalized document claim, source-topic alignment, matched source
terms, restriction and gluing-diagnostic counts, and any source-coherence
warning.  The local model is asked to emit only a visible Toulmin ticket:
claim, grounds, warrant, qualifier, and rebuttal.  \textsc{Odyssey} then compares the
ticket with the grounded packet, storing the result in
\texttt{local\_llm\_grounded\_toulmin\_comparison.json} and rendering it through
a Scylla page.

Before adding the neural inspector, we ran seven deterministic Toulmin
certification cases across the current \textsc{Odyssey} domains:
Indus uncertainty, TRACE science reasoning, TRACE math reasoning, TCC causal
claims, 10-K filing arguments, brand-product feedback, and Democritus causal
claims.  Table~\ref{tab:toulmin-certification-runs} summarizes the result.  In
all seven cases the foundry built a full Toulmin ticket with the required
claim, grounds, warrant, qualifier, and rebuttal roles, plus a warrant bridge
and argument hyperedge.  None of the seven was silently promoted.  Each retained
an explicit rebuttal obstruction and was quarantined for review, which is the
desired behavior for a certification layer: Toulmin certifies that the argument
surface is complete enough to inspect, not that the claim has become globally
true.

\begin{table}[t]
\centering
\small
\begin{tabular}{p{0.20\linewidth}p{0.22\linewidth}p{0.18\linewidth}p{0.28\linewidth}}
\toprule
Case & Target foundry & Result & Retained obstruction \\
\midrule
Indus uncertainty & Scientific challenge & 5/5 roles; review &
Decipherment remains unsettled. \\
TRACE science & Toulmin evaluation baseline & 5/5 roles; review &
Thermodynamic claim requires kinetic-condition exception surface. \\
TRACE math & Toulmin evaluation baseline & 5/5 roles; review &
Solution claim remains scoped by the stated algebraic constraints. \\
TCC causal claims & Causal claims foundry & 5/5 roles; review &
Income--health edge remains scoped to the selected causal-claim evidence. \\
10-K filings & Institutional financial & 5/5 roles; review &
Margin-pressure claim remains scoped to filing language and risk context. \\
Brand reviews & Market meaning & 5/5 roles; review &
Comfort claim remains product- and source-scoped. \\
Democritus & Causal claims foundry & 5/5 roles; review &
Microplastics edge remains candidate under provenance and support limits. \\
\bottomrule
\end{tabular}
\caption{Deterministic Toulmin certification cases.  All seven runs
constructed the required role coverage, one warrant bridge, and one argument
hyperedge.  All seven also retained an explicit rebuttal obstruction, so the
promotion decision was \texttt{quarantine\_for\_review} rather than admission
without exception surface.}
\label{tab:toulmin-certification-runs}
\end{table}

\begin{table}[t]
\centering
\small
\begin{tabular}{p{0.24\linewidth}p{0.18\linewidth}p{0.18\linewidth}p{0.28\linewidth}}
\toprule
Local model & Scale role & Review cases & Shared obstruction \\
\midrule
Llama~3.2--3B MLX & Small baseline & \(1/6\) &
Dinosaur Extinction radiometric-dating claim with
\texttt{claim\_mismatch} and \texttt{source\_alignment\_review}. \\
Qwen3-Next-80B MLX & Larger local model & \(1/6\) &
Same Dinosaur Extinction packet, with the source-alignment review signal
preserved. \\
Qwen3-235B MLX & Largest local model & \(1/6\) &
Same Dinosaur Extinction packet, again with claim-preservation and
source-alignment review active. \\
\bottomrule
\end{tabular}
\caption{Scale-independent obstruction in the grounded Toulmin pilot.  \textsc{Odyssey}
ran the same six Democritus/Prometheus grounded claims through three local MLX
models.  The models differed in ticket style and verbosity, but all converged
on the same Scylla review case, indicating that the obstruction is tied to the
sheaf-grounded packet rather than to a particular local model.}
\label{tab:grounded-toulmin-model-convergence}
\end{table}

This implementation makes Toulmin a neural argument inspector rather than a
truth oracle.  It checks whether the model preserves the extracted document
claim and whether its warrant honestly respects source alignment, restrictions,
gluing diagnostics, and rebuttals.  Table~\ref{tab:grounded-toulmin-model-convergence}
reports the first scale-sensitivity check: on the same six
Democritus/Prometheus grounded claims, Llama~3.2--3B, Qwen3-Next-80B, and
Qwen3-235B all exposed the same Dinosaur Extinction review packet.  This is the
intended failure mode.  Scylla does not merely display the model's argument,
but compares the argument against the sheaf-grounded packet that licensed the
prompt.

\section{\textsc{Odyssey} as Algebraic Headquarters}
\label{sec:odyssey-headquarters}

The earlier sections introduced \textsc{Odyssey}'s foundry algebra and admission logic.
The system-level point is that \textsc{Odyssey} is not a renamed Prometheus run.  It is
the algebraic headquarters that decides which representational object is being
built, which local contexts must exist, which overlap laws must be checked, and
which artifacts may enter durable state.  Prometheus remains crucial, but it is
called through typed contracts rather than used as an ungoverned side channel.

Figure~\ref{fig:pipeline} shows the current \textsc{Odyssey} pipeline.  A human request
first becomes a Scylla-facing intent and answer contract.  Homer compiles that
intent into an executable workflow and, when needed, a Prometheus job contract.
Athena type-checks the foundry expression: cover, local representation family,
truth values, restriction maps, gluing rule, bridges, and obstruction policy.
Prometheus then builds or refreshes the requested artifact bundle, including
BRIDGE/SKFM refinement screens and IDC-style local causal diagnostics when the
target foundry calls for causal admission.  Toulmin constructs the argument
ticket: claims, grounds, warrants, backing, qualifiers, rebuttals, and
applicable contexts.  Finally, \textsc{Odyssey} runs TICKET admission against
the target foundry, promoting compatible state and quarantining or blocking
non-gluing candidates; SkillOpt can use the resulting admission traces to
optimize the natural-language policy that produced the next candidate ticket.

\begin{figure}[t]
\centering
\footnotesize
\begin{tikzpicture}[
  node distance=6mm,
  box/.style={draw,rounded corners,align=center,text width=0.86\linewidth,inner sep=5pt},
  arrow/.style={-Latex,thick}
]
\node[box] (request) {User request and domain goal};
\node[box,below=of request] (scylla) {\textsc{Scylla}: foundry intent, answer contract, responsible explanation};
\node[box,below=of scylla] (homer) {\textsc{Homer}: workflow skeleton, replay obligations, Prometheus job contract};
\node[box,below=of homer] (athena) {\textsc{Athena}: cover, local representation types, restrictions, gluing rule, obstruction policy};
\node[box,below=of athena] (prom) {\textsc{Prometheus}: source ingestion, local sections, audits, dashboards, artifact bundle, BRIDGE/SKFM screens};
\node[box,below=of prom] (toulmin) {\textsc{Toulmin}: construct warranted claims, qualifiers, rebuttals, and argument tickets};
\node[box,below=of toulmin] (ticket) {\textsc{TICKET}: restrict to target foundry, run IDC-style local causal diagnostics, promote or quarantine};
\node[box,below=of ticket] (skillopt) {SkillOpt: optimize admission skill from rollout, reflection, and held-out gate traces};
\node[box,below=of skillopt] (state) {Durable \textsc{Odyssey} foundry state with Scylla-facing views and refresh contracts};
\draw[arrow] (request) -- (scylla);
\draw[arrow] (scylla) -- (homer);
\draw[arrow] (homer) -- (athena);
\draw[arrow] (athena) -- (prom);
\draw[arrow] (prom) -- (toulmin);
\draw[arrow] (prom) -- (ticket);
\draw[arrow] (toulmin) -- (ticket);
\draw[arrow] (ticket) -- (skillopt);
\draw[arrow] (skillopt) -- (state);
\end{tikzpicture}
\caption{\textsc{Odyssey} turns a user request into maintained foundry state
through typed Scylla, Homer, Athena, Prometheus, Toulmin, TICKET, and
SkillOpt contracts.  In causal foundries, Prometheus may emit BRIDGE/SKFM
residual-obstruction artifacts, TICKET may apply IDC-style local causal
diagnostics, and SkillOpt can optimize the admission skill using the resulting
rollout and gate traces.}
\label{fig:pipeline}
\end{figure}
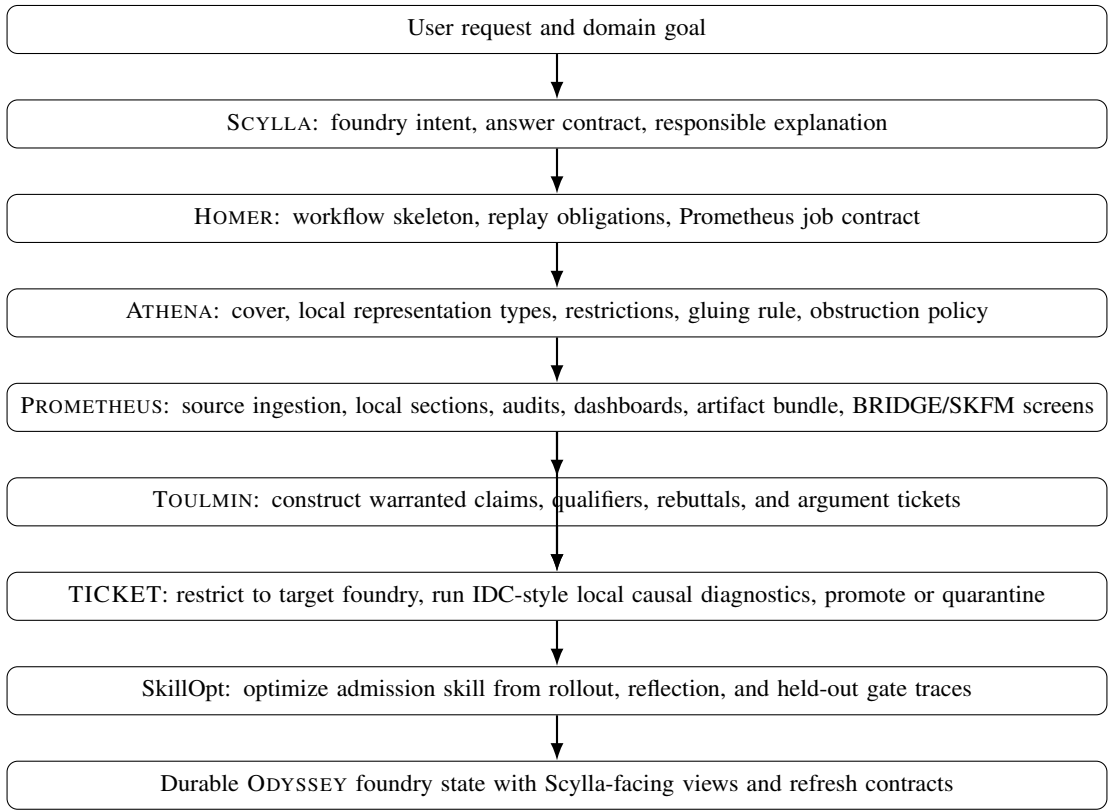

\paragraph{\textsc{Odyssey} state.}
The promoted object is not only a generated HTML page or a model snapshot.  It
is a typed foundry state: local sections, source sheaves, restriction maps,
gluing records, obstruction ledgers, promotion decisions, refresh obligations,
and user-facing views.  This is why examples such as Dick's Sporting Goods,
Amazon Reviews 2023, Indus Script, MyFixIt, TCC 44K, research-program, and
assistant-build can share one architecture while retaining different local
semantics.

\paragraph{Contract discipline.}
\textsc{Odyssey}'s main systems constraint is that every backend artifact must cross a
typed interface.  A Prometheus run may propose a source manifest, world model,
gluing audit, dashboard, model manifest, or refresh plan.  TICKET then decides
whether those objects restrict into the target foundry and whether their
overlaps glue with maintained state.  This turns ingestion from an append-only
file drop into an algebraic admission problem.

\section{\textsc{Athena}: Representation Contracts}
\label{sec:athena}

Athena is the module that prevents \textsc{Odyssey} from collapsing into ordinary
retrieval.  It decides what kind of object the system is allowed to build.  In
the current implementation, a Homer program is compiled into
\texttt{athena\_sheaf\_plan.json}, which records a representation strategy, a
finite truth-value strategy, a cover of local contexts, overlap declarations,
gluing logic, artifact strategy, source sheaves, and cross-sheaf bridges.

The key Athena output is a cover.  For a corporation model, local contexts
include financial statements, risk factors, management discussion, strategy
signals, and market context.  For an Indus Script model, they include symbol
inventory, inscription sequences, statistical structure, decipherment
hypotheses, archaeological context, and controversy uncertainty.  For an
assistant-build foundry, they include users, tasks, tool permissions, memory
boundary, autonomy level, policy, evals, simulations, deployment surface, and
maintenance cadence.  These covers are not labels for a dashboard; they define
where predicates may be evaluated and where transport is allowed.

Athena also owns the finite truth semantics used by the deterministic validation
artifacts:
\[
  \bot,\ \mathrm{WEAK},\ \mathrm{PLAUSIBLE},\ \mathrm{SUPPORTED},\ \top.
\]
This partition gives \textsc{Odyssey} a small, inspectable substitute for uncalibrated
confidence prose.  An overlap glues when shared predicates are at least
plausible and no paired predicate is bottom.  If the overlap does not glue,
Athena's contract requires a Prometheus obstruction record: which predicates
failed, which evidence was shared, and what next observation would be needed.

Athena is therefore the type checker for foundry construction.  It specifies
the local representation family before Prometheus builds artifacts, and it
specifies the gluing law before \textsc{Odyssey} is allowed to promote an artifact into
maintained state.

\section{\textsc{Homer}: Workflow and Replay Contracts}
\label{sec:homer}

Homer is the orchestration layer.  Given a Scylla brief, it emits
\texttt{model.homer.json}: a workflow skeleton, required contexts, required
audits, Athena requests, and Prometheus build steps.  The current default
workflow parses the user goal, constructs a candidate cover, requests Athena's
sheaf plan, builds a Prometheus world model, runs gluing audits, and serves a
Scylla explanation.

Homer becomes especially important for process foundries.  A research-program
foundry needs source manifests, goals, hypotheses, methods, experiments,
evidence, artifact registries, refresh cadence, and obstruction policies.  An
assistant-build foundry needs tool-permission manifests, eval suites,
simulation traces, failure-mode analysis, deployment surfaces, model cards, and
maintenance plans.  An evaluation-harness foundry needs dataset restrictions,
task lanes, embedding adapters, query-document objects, metric protocols,
failure slices, and promotion gates.  Homer records these obligations as a
program rather than as prose.

When Prometheus is required, Homer compiles a job contract rather than issuing
an informal backend request.  The contract names the target foundry, source
family, collection plan, ingestion adapter, expected artifacts, budget limits,
replay metadata, and refresh cadence.  This lets a future run be compared with
the original build and lets TICKET decide whether new artifacts should replace,
extend, or quarantine earlier state.

\section{\textsc{Prometheus} Inside \textsc{Odyssey}}
\label{sec:prometheus-inside-odyssey}

The Prometheus described in the earlier arXiv paper built causal Topos World
Models from research corpora.  Inside \textsc{Odyssey}, Prometheus has a narrower and
more operational role.  It is the engine room that materializes Athena's plan:
local sections, predicate values, restriction maps, overlap results,
integration layers, update status, evidence sources, HTML explorers,
Scylla-facing dashboards, ingestion packets, and callback bundles.

This newer Prometheus is artifact-first.  It returns a bundle rather than a
trusted foundry:
\[
\begin{gathered}
  \texttt{source\_manifest},\quad
  \texttt{artifact\_manifest},\quad
  \texttt{prometheus\_world\_model.json},\\
  \texttt{gluing\_audit.json},\quad
  \texttt{obstruction\_ledger},\quad
  \texttt{dashboards},\quad
  \texttt{refresh\_plan}.
\end{gathered}
\]
\textsc{Odyssey} then performs admission.  A bundle may be promoted when it restricts to
the target foundry and glues with maintained state.  It may be quarantined when
it is useful as a source atlas but not yet part of the durable foundry.  It may
be blocked when provenance, typing, or overlap compatibility fails.

Prometheus also remains the place where heavier evidence construction happens.
It can ingest papers, filings, reviews, repair manuals, CSQL/DuckDB atlases,
benchmark outputs, simulation traces, and prior Prometheus runs.  It can build
domain-specific source sheaves such as shopping-experience, corporate-workflow,
Amazon Reviews 2023, Indus evidence, MyFixIt procedural PSRs, and TCC causal
claim atlases.  What changes in \textsc{Odyssey} is governance: Prometheus proposes
evidence-bearing world-model artifacts; \textsc{Odyssey} decides, through Athena's laws
and TICKET admission, what becomes maintained foundation-model state.

\subsection{Bidirectional \textsc{Odyssey}--\textsc{Prometheus} Transport}
\label{sec:odyssey-prometheus-transport}

The current implementation is bidirectional.  In the forward direction, \textsc{Odyssey}
can import an existing Prometheus run as a candidate model.  The importer
inspects a Prometheus output directory, loads its world model, maps the source
family to a target foundry, synthesizes a Homer-style job contract when one is
missing, builds a TICKET admission record, and exposes the result to Scylla
through an ingestion console.  The current console records \(129\) inspected
Prometheus-family cases: \(6\) admitted \textsc{Odyssey}-native Prometheus bundles, \(6\)
queryable but quarantined bundles with preserved obstruction ledgers, and
\(117\) legacy Prometheus v1 run-file candidates that remain visible but are
not treated as durable \textsc{Odyssey} state until TICKET review is complete.  This is
an important empirical point: \textsc{Odyssey} already treats roughly one hundred
pre-existing Prometheus models as typed candidates rather than as opaque files
or trusted global models.

In the reverse direction, \textsc{Odyssey} can compile a foundry back into a Prometheus
GUI-compatible Topos World Model layout.  The compiler projects \textsc{Odyssey} local
sections into Prometheus contexts, converts source-sheaf events into Prometheus
episodes, emits local PSRs and sheaf objects, writes restrictions and gluing
diagnostics, and preserves the \textsc{Odyssey} ingestion packet and source artifacts.
The folder layout is deliberately compatible with the Prometheus atlas
workbench, while the manifest records the actual \textsc{Odyssey} source family and
target foundry.  For DKS, this reverse path exports the sporting-goods retailer
foundry as a DKS Prometheus run with shopping-experience and
corporate-workflow sheaves, a hidden-state-capture diagnostic, gluing
diagnostics, and a health-history record.  The point is not to abandon
\textsc{Odyssey}'s admission discipline, but to use Prometheus as a visual microscope:
we can inspect the sheaf structure, local PSRs, restrictions, and gluing
tensions explicitly, then carry the annotated tensions back into \textsc{Odyssey}'s
promotion ledger.

\begin{table}[t]
\centering
\small
\begin{tabular}{p{0.23\linewidth}p{0.29\linewidth}p{0.35\linewidth}}
\toprule
Direction & Transport contract & Empirical artifact \\
\midrule
Prometheus \(\to\) \textsc{Odyssey} & Load world model, infer source family and target
foundry, build TICKET packet, decide admitted/quarantined/candidate. & Scylla
ingestion console over \(129\) cases: \(6\) admitted, \(6\) quarantined,
\(117\) legacy candidates awaiting review. \\
\textsc{Odyssey} \(\to\) Prometheus & Project local sections to Prometheus contexts,
events, local PSRs, sheaf objects, restrictions, and gluing diagnostics. & DKS
and MyFixIt exports into Prometheus-compatible atlas layouts for visual sheaf
inspection and tension tagging. \\
\bottomrule
\end{tabular}
\caption{Bidirectional transport between \textsc{Odyssey} and Prometheus.  Import keeps
legacy Prometheus models behind TICKET gates; export uses Prometheus as an
inspection surface for \textsc{Odyssey} foundries.}
\label{tab:odyssey-prometheus-transport}
\end{table}

\section{Prometheus PSR Interface}
\label{sec:text-to-psr}

\textsc{Odyssey} uses \textsc{Prometheus} as a world-model construction engine, but the
full language-to-PSR derivation belongs to the Prometheus paper
\citep{mahadevan2026prometheus}.  Here we only need the interface that \textsc{Odyssey}
receives from Prometheus and admits through TICKET.  A Prometheus artifact
presents a finite family of local predictive-state sections over a declared
cover.  Each local section stores histories, tests or predictive motifs,
prediction/support values, provenance, diagnostics, restriction maps, gluing
results, and obstruction records.

In \textsc{Odyssey} terms, the important object is not the estimator that produced the
local PSR tables, but the typed contract they satisfy.  For each context
\(U\), Prometheus emits a local section of the schematic form
\[
  \Psh(U)=(H_U,T_U,M_U,S_U,\Pi_U,D_U),
\]
where \(H_U\) and \(T_U\) are the finite histories and tests visible in that
context, \(M_U\) is the predictive table or support score, \(S_U\) records
support, \(\Pi_U\) records provenance, and \(D_U\) records diagnostics such as
sparsity, uncertainty, extraction confidence, and local mismatch.  A morphism
\(V\to U\) carries a restriction map that projects histories, tests, claims,
and provenance from \(U\) to the overlap visible in \(V\).  Gluing diagnostics
then compare the restricted local sections on shared signatures.

This is the only PSR machinery \textsc{Odyssey} needs for foundation-model construction.
Scylla and Homer do not depend on a particular spectral estimator; Athena only
needs the cover, restriction, and gluing law; TICKET only needs the admission
packet, compatibility result, and obstruction ledger; and Toulmin only needs the
claim, grounds, warrant, qualifier, and rebuttal that can be read from the
maintained local state.  A Prometheus artifact can therefore be treated as a
source foundry candidate: it is promoted when its local sections restrict and
glue under the target foundry contract, quarantined when useful but not yet
compatible, and blocked when the obstruction ledger shows missing provenance,
failed overlap compatibility, or unsupported transport.

The operational sheaf condition used by \textsc{Odyssey} is deliberately finite.  Local
sections over a cover may be promoted only when their restrictions agree on
supported shared cells up to the declared tolerance.  Unsupported cells remain
local, and incompatible cells are preserved as obstruction records rather than
averaged away.  This is what lets \textsc{Odyssey} use Prometheus as an engine-room
component without making this paper repeat the Prometheus construction: the
vision paper studies how such artifacts become durable foundation-model state,
not how the underlying local PSR estimator is derived.

\section{Finite Guarantees from Sheaf Foundries}
\label{sec:sheaf-guarantees}

The sheaf formalism is not only a visualization language.  It gives \textsc{Odyssey} a
small set of finite guarantees that are useful for foundation-model
construction.  The guarantees are conditional on the declared cover,
restriction maps, support thresholds, and truth partition; they do not assert
that extracted claims are true in the external world.  They assert that \textsc{Odyssey}
will only promote global state when the local evidence can be transported and
glued under the declared contract.

\begin{definition}[Foundry health]
Let \(\mathcal{U}=\{U_i\to U\}\) be a foundry cover and let each local PSR
section \(s_i\) expose a finite matrix of truth-weighted predictive cells
\(M_i[h,\tau]\in[0,1]\).  Fix a promotion threshold \(\theta\), corresponding
in the implementation to the truth label \textsc{PLAUSIBLE}.  The
cell-level hidden-state-capture ratio is
\[
  H_{\mathrm{cell}}(\mathcal{U})
  =
  \frac{\sum_i |\{(h,\tau):M_i[h,\tau]\ge \theta\}|}
       {\sum_i |\mathrm{dom}(M_i)|}.
\]
The context-level capture ratio is
\[
  H_{\mathrm{ctx}}(\mathcal{U})
  =
  \frac{|\{i:H_{\mathrm{cell}}(U_i)\ge 0.6\}|}{|\mathcal{U}|}.
\]
\textsc{Odyssey} reports \emph{stable hidden-state capture} when both ratios are at
least \(0.6\), \emph{partial hidden-state capture} when the global cell ratio is
at least \(0.35\), and \emph{weak hidden-state capture} otherwise.
\end{definition}

This number is deliberately simple.  It measures whether the local predictive
tables have enough plausible-or-better cells to serve as a maintained hidden
state proxy.  It is not a final task score; it is a health diagnostic for the
sheaf representation.  In the Prometheus export path, for example, DKS and
MyFixIt receive a hidden-state-capture diagnostic before their sheaf structure
is inspected visually.

\begin{theorem}[Consistent promotion]
Fix a foundry cover \(\mathcal{U}\), finite local sections \(s_i\), restriction
maps \(\rho_{ij}\), support weights, and tolerance \(\epsilon\).  If every
overlap has support above threshold and
\[
  \|\rho_{ij}(s_i)-\rho_{ji}(s_j)\|\le \epsilon
\]
on every shared predictive cell, then \textsc{Odyssey}'s support-weighted aggregation
constructs a promoted section \(s\) whose restrictions agree with every local
section up to \(\epsilon\) on supported overlaps.
\end{theorem}

\begin{proof}
The implementation forms each promoted cell only from local cells that pass the
declared support and compatibility checks.  The promoted value is a
support-weighted average of those compatible values.  Since each contributing
local value differs from every other contributing local value by at most
\(\epsilon\), the average also lies within the same \(\epsilon\)-diameter
interval on the overlap.  Unsupported cells are not promoted, so no unsupported
agreement is asserted.
\end{proof}

\begin{proposition}[Obstruction soundness]
If an overlap fails compatibility, lacks a restriction map, or lacks sufficient
support, \textsc{Odyssey} cannot silently promote the corresponding global claim through
TICKET.  The failed overlap must appear as a gluing diagnostic, obstruction
ledger entry, quarantine reason, or blocked promotion gate.
\end{proposition}

\begin{proof}
TICKET admission reads the gluing audit and promotion status emitted by
Prometheus.  The admission record is green only when required artifacts are
present and preserved obstructions are empty for the target promotion lane.  A
failed overlap therefore changes the admission state from admitted to
quarantined, candidate, or blocked, and the obstruction is retained with the
affected contexts and next observation.
\end{proof}

\begin{proposition}[Non-transportability certificate]
Suppose a claim \(c\) is local to \(U_i\).  If there is no declared restriction
path from \(U_i\) to a target context \(V\), or if every such path contains an
overlap whose tension exceeds tolerance, then \textsc{Odyssey} cannot promote \(c\) as a
claim over \(V\).  The best available status is local, quarantined, or awaiting
new evidence.
\end{proposition}

\begin{proof}
Promotion over \(V\) requires a chain of restriction maps that transports the
claim's histories, tests, support, and provenance into \(V\), followed by a
compatible gluing check.  If the chain is missing, the claim has no typed
meaning in \(V\).  If the chain exists but contains a high-tension overlap, the
operational sheaf condition fails on that path.  In either case, TICKET lacks a
valid promoted section over \(V\).
\end{proof}

\begin{proposition}[Health monotonicity under compatible refinement]
Let a foundry be refined by adding new local cells or new local sections whose
truth weights are all at least the current threshold \(\theta\), and whose
overlaps are compatible with existing promoted sections.  Then
\(H_{\mathrm{cell}}\) cannot decrease.  If the added sections also satisfy the
context capture threshold, \(H_{\mathrm{ctx}}\) cannot decrease.
\end{proposition}

\begin{proof}
The cell ratio is a fraction whose numerator and denominator both increase by
the number of added cells, because every added cell is above threshold.  Such an
update cannot lower the ratio.  The same argument applies to context capture
when every added section is captured; compatible overlaps ensure that the added
cells are eligible for the same foundry state rather than being routed to a
separate obstruction ledger.
\end{proof}

\begin{corollary}[Auditability of model health]
Changes in foundry health are attributable.  A decrease can only arise from
adding weak cells, adding weak contexts, revising truth weights downward, or
splitting previously compatible material into an obstructed local section.
\end{corollary}

These guarantees explain the role of sheaves in the empirical sections.  DKS
can be promoted because customer, assortment, brand, and filing sections glue.
Indus decipherment claims remain obstructed because the controversy/hypothesis
overlap does not support a settled translation.  KET can make corpus-local
claims while blocking a global KET-versus-Transformer claim when the WikiText-2
KET rows are missing.  MyFixIt improves retrieval through structured
action-observation state while keeping visual grounding outside promotion until
image evidence is fetched or checksummed.

\paragraph{Two-stage gluing.}
The same formalism can be iterated across levels of analysis.  In the filing
experiments, the local workflow slices
\[
  x_{C,y}^{\mathrm{ops}},\quad
  x_{C,y}^{\mathrm{mkt}},\quad
  x_{C,y}^{\mathrm{fin}},\quad
  x_{C,y}^{\mathrm{inn}}
\]
may first glue into a company-year section \(s_{C,y}\).  These company-year
sections can then be compared over a second cover, such as sector-year or
temporal-neighborhood covers.  The important point is that there is no single
unconditional global average over all firms, papers, reviews, or agents.
Globality is always relative to a declared cover, and non-gluing at a coarser
cover is a meaningful result.

\paragraph{Localized interventions.}
\textsc{Prometheus} treats interventions as local tests.  A \(j\)-do query
\(\doop_j(X=x)\) modifies histories or tests inside a context \(U\) and asks how
the local predictive-state table changes under comparable covers.  The result
is not automatically an identified causal effect in Pearl's sense
\citep{pearl2009causality}; it is an intervention-conditioned probe of the
language-derived world model.  This convention is also the point at which our
recent IDC/infinitesimal-causality work becomes relevant: the local query is
best read as a typed deformation of predictive state, with support, provenance,
and local variation recorded before any global causal effect is asserted
\citep{mahadevan2026infinitesimalcausality}.  \textsc{Prometheus} reports the
support and provenance behind the probe rather than presenting it as a
source-free causal estimate.

More explicitly, let
\[
  j(U)=\{u_i:U_i\to U\}
\]
be a cover of contexts considered comparable for the query, and let
\[
  I^{a}_{U_i}:\Psh(U_i)\to \Psh^{\doop(a)}(U_i)
\]
be a local intervention map that edits a test, fixes an action, inserts a repair
step, or conditions on an explicitly declared regime.  The \(j\)-localized
intervention state is computed by restriction, local intervention, and
aggregation:
\[
  \doop_j(a)_U(s)
  =
  \operatorname{Agg}_{u_i:U_i\to U\in j(U)}
  \left(
    I^{a}_{U_i}\bigl(\rho_{U,U_i}(s)\bigr)
  \right).
\]
Compatibility is then checked after the intervention.  If the intervened local
sections glue, the atlas may report a coherent intervention-conditioned
prediction over \(U\).  If they do not, the query is only locally supported, and
the failed overlaps identify where comparability, measurement, or evidence
breaks down.

\section{The Claims Atlas}

The primary user-facing object is the Claims Atlas.  It is designed to answer
research questions that flat summaries obscure.

\paragraph{Main causal spine.}
The atlas extracts recurrent, high-support causal paths that organize the
corpus.  In an ocean-warming corpus, a spine may include warming, stratification,
oxygen loss, prey availability, migration, recruitment, and population change.
In SEC workflows, a spine may include investment, supply-chain constraints,
margin pressure, capital allocation, and realized outcomes.

\paragraph{Local context regions.}
Each spine is decomposed into local regions.  A region may correspond to a
species group, geography, time period, document cluster, product aspect, or
workflow stage.  Users can enter a region and inspect its local PSR, support,
claims, and provenance.

\paragraph{Drift detection.}
When local models change across time, retrieval runs, or document strata,
\textsc{Prometheus} reports drift.  Drift can be textual, causal, predictive, or
topological: the support distribution changes, a causal polarity changes, a
test prediction changes, or the overlap graph itself changes.

\paragraph{Regime tensions.}
The atlas highlights where local models resist gluing.  Some tensions are
contradictions; others are legitimate regime boundaries.  The interface should
make this distinction visible by exposing modifiers, populations, measurement
protocols, and source provenance.

\paragraph{Provenance drill-downs.}
Every atlas claim points back to evidence units.  A user can inspect source
passages, extracted rows, normalized claims, support counts, and neighboring
contexts.  Provenance is not decorative metadata; it is the mechanism by which
the atlas remains corrigible.

\section{Implemented \textsc{Odyssey} Foundries}
\label{sec:implemented-foundries}

The current \textsc{Odyssey} repository contains a family of concrete foundry instances
rather than a single monolithic demo.  Each instance follows the same artifact
contract: Scylla emits a model brief, Homer emits a workflow skeleton, Athena
emits a sheaf plan, and Prometheus emits a world model, gluing audit, and
human-facing explorer.  This section summarizes the foundries that have been
implemented so far and the role each plays in testing the \textsc{Odyssey} design.

\begin{table}[t]
\centering
\small
\begin{tabular}{p{0.20\linewidth}p{0.31\linewidth}p{0.38\linewidth}}
\toprule
Foundry & Cover / local sections & Current purpose \\
\midrule
Storefront operations & Sales demand, inventory flow, staffing capacity,
pricing margin, customer experience. & A compact operational-decision foundry
for testing local predicate gluing around staffing, reordering, promotion, and
service recovery decisions. \\
Brand / product meaning & Customer journey, product experience, brand
promise, channel message, competitive context. & A market-meaning foundry for
testing whether product evidence, reviews, and promise claims cohere. \\
Corporation & Financial statements, risk factors, management discussion,
strategy signals, market context. & An institutional-financial foundry for
10-K-style evidence, narrative/number agreement, and risk/strategy bridges. \\
Dick's Sporting Goods & Review themes, store experience, product assortment,
brand promise, corporate context. & A domain restriction of the generic
sporting-goods retailer expression: operational decision, market meaning,
institutional-financial evidence, and review evidence, with cross-sheaf bridges
from customer evidence to DKS filing evidence. \\
Amazon Reviews 2023 & Corpus manifest, recommendation benchmarks,
language-item models, product-search tasks, reproduction contract. & A
corpus-benchmark foundry that keeps dataset schema, benchmark splits,
BLaIR-style language-item modeling, and replay obligations local. \\
KET language-modeling evaluation & Paper/code bundle, PTB, WikiText-2,
WikiText-103 restrictions, model variants, metrics, diagnostics, replay gaps. &
A restriction of the generic research-program and evaluation-harness foundries
to the KET language-modeling suite, preserving corpus-local promotion gates
instead of flattening PTB, WikiText-2, and WikiText-103 into one leaderboard. \\
Testing Causal Claims 44K & Causal-node registry, aggregate causal edges,
edge-support rows, causal direction, method/journal/\(p\)-value lookup tables,
uncertainty guardrails. & A large-scale economics foundry over the TCC corpus:
\textsc{Odyssey} restricts a generic causal-literature atlas to \(44\)K papers while
keeping provenance, method, polarity, reverse-causality, and controversy
constraints explicit. \\
Indus Script & Symbol inventory, inscription sequences, statistical
structure, decipherment hypotheses, archaeological context, controversy. & A
scientific-challenge foundry that preserves rival hypotheses instead of
collapsing uncertainty into a single decipherment claim. \\
MyFixIt & Manual manifest, step text, tools, parts, removal verbs, images,
future-step tests. & A procedural PSR foundry over repair manuals, used as the
first evaluation-harness restriction. \\
IKEA ASM & Multi-view RGB/depth videos, action labels, furniture parts,
instance segmentation, part tracks, 2D/3D pose, camera calibration. & A
multimodal procedural assembly foundry: the generic procedural PSR and
perception-action evaluation foundries restricted to furniture assembly
episodes. \\
Research program / assistant / evaluation harness & Goals, tasks, methods,
tools, memory, policy, metrics, failure slices, refresh cadence. & Process
foundries for maintaining \textsc{Odyssey} itself and for comparing PSR-style
representations with token baselines. \\
Grounded Toulmin/local LLM & Prometheus event, document claim, source context,
source alignment, restriction/gluing diagnostics, visible Toulmin ticket. & A
neural argument-inspection foundry for comparing local LLMs on whether they
preserve grounded claims and warrant them without over-transporting weak
evidence. \\
\bottomrule
\end{tabular}
\caption{Implemented \textsc{Odyssey} foundry families.  Each row is backed by generated
JSON artifacts and HTML views in the repository, not only by a conceptual
taxonomy.}
\label{tab:implemented-foundries}
\end{table}

\paragraph{Common artifact contract.}
The repeated pattern across these foundries is more important than any one
domain.  Each run records the source request, a typed target foundry, the local
cover, overlap checks, finite truth values, gluing outcomes, and update status.
The same surface can therefore represent a store operations problem, a filing
workflow, an archaeological controversy, or a repair-manual procedure.  The
system's current deterministic templates are intentionally simple; their role
is to make the foundry interfaces replayable while richer extraction,
retrieval, and learned local models are still under development.

\paragraph{Generic and specialized foundries.}
Several examples test generic representational foundries directly: the
storefront foundry tests operational decisions, the brand foundry tests market
meaning, the corporation foundry tests institutional-financial evidence, and
the evaluation harness tests metric protocols.  Specialized foundries compose
these generic objects and then restrict them to a concrete domain, corpus, or
task lane.  Dick's Sporting Goods is the restriction of a retail/brand/review
and institutional-financial composition to DKS: \textsc{Odyssey} glues a
shopping-experience sheaf to corporate workflow and 10-K evidence, then checks
whether store-experience, assortment, brand-promise, inventory, market, and risk
signals agree.  The KET language-modeling foundry is the restriction of the
generic research-program and evaluation-harness foundries to the KET experiment
suite.  It keeps PTB, WikiText-2, and WikiText-103 as separate corpus slices:
PTB and WikiText-103 admit near-tied KET/Transformer denoising rows, WikiText-2
blocks KET-vs-Transformer promotion because KET rows are absent from the
admitted table, and full replay remains deferred until checkpoint checksums,
hardware/runtime metadata, and command logs are attached.  Amazon Reviews 2023
glues corpus, benchmark, model, search, and reproduction contracts.  TCC 44K
restricts a generic causal-claims atlas to the economics literature and glues
nodes, edges, support rows, method/journal/\(p\)-value lookup tables, and
uncertainty guardrails while keeping direction and polarity tensions visible.
Indus Script glues visual, sequential, statistical, archaeological, and
controversy contexts while preserving undecidable regions as explicit
obstructions.  IKEA ASM is the
restriction of a generic procedural PSR and perception-action foundry to
furniture assembly: actions, parts, pose, camera geometry, and object tracks
must agree on the same episode rather than being evaluated as unrelated
computer-vision tasks.

The following subsections give the artifact-backed case studies for the
implemented foundries rather than introducing new top-level paper themes.

\subsection{Scientific Foundries: TCC 44K and Indus Script}
\label{sec:scientific-foundries}

TCC 44K and Indus Script are useful because they stress the same foundry
contract in opposite regimes.  TCC 44K is a high-volume map of the study of
causality in economics, grounded in the Testing Causal Claims corpus
\citep{gargFetzer2025testingCausalClaims}.  The local CSQL atlas records
\(295{,}252\) canonical cause/effect nodes, \(261{,}714\) aggregate causal
edges, and \(265{,}656\) support rows drawn from a roughly \(44\)K-paper
corpus.  Its sheaf cover keeps node identities, aggregate edges, document
support, causal direction, method/journal/\(p\)-value lookup tables, and
uncertainty guardrails separate.
The important point is not only scale.  A repeated claim such as monetary
policy increasing inflation can glue through node-edge and edge-support
provenance, but \textsc{Odyssey} still refuses to turn support mass into causal
certainty when polarity, controversy, or reverse-causality evidence remains
weak.

Indus Script sits at the other extreme.  The artifact records \(419\) indexed
sign images, \(1{,}548\) visual inscription sequences, a small decipherment
literature sheaf, anchored by Parpola's authoritative study of the corpus and
decipherment problem \citep{parpola2009decipheringIndus}, and archaeological
context including collapse-era figure data.  Here the central risk is not
leaderboard collapse but premature interpretation.  Recent computational work
argues that the Indus sign system does not cleanly match either heraldic or
administrative non-linguistic baselines, while still preserving the central
uncertainty about whether it encodes spoken language
\citep{nair2026indusNonLinguistic}.  Symbol inventories glue to inscription
sequences, sequences glue to \(n\)-gram and entropy statistics, and statistical
structure can inform a hypothesis lattice.  The
controversy/hypothesis overlap deliberately obstructs: language-like sequential
structure, archaeological plausibility, and candidate language-family priors do
not by themselves license a trusted translation in the absence of bilingual
anchors or stronger external evidence.

\begin{table}[t]
\centering
\footnotesize
\setlength{\tabcolsep}{4pt}
\begin{tabular}{p{0.16\linewidth}p{0.28\linewidth}p{0.23\linewidth}p{0.20\linewidth}}
\toprule
Foundry & Scale and local sections & What glues & What remains blocked \\
\midrule
TCC 44K & \(44\)K-paper economics causal-claims atlas; \(295{,}252\) nodes,
\(261{,}714\) edges, \(265{,}656\) support rows, method/journal/\(p\)-value
lookup tables. & Node-edge, edge-support, and support-to-lookup restrictions
preserve document provenance and evidence summaries. & Direction, polarity,
controversy, and reverse-causality joins remain reviewable guardrails. \\
Indus Script & \(419\) signs, \(1{,}548\) visual sequences, decipherment
papers, statistical structure, archaeological context, and controversy
records. & Symbol-sequence, sequence-statistics, statistics-hypotheses, and
hypotheses-archaeology restrictions glue provisionally. & The
controversy/hypothesis restriction blocks any settled decipherment claim. \\
\bottomrule
\end{tabular}
\caption{Two scientific foundries illustrate why \textsc{Odyssey} treats gluing as an
epistemic discipline.  TCC prevents a large support graph from becoming an
unqualified causal truth map; Indus prevents suggestive statistical structure
from becoming a premature translation.}
\label{tab:scientific-foundries}
\end{table}

\subsection{MyFixIt Procedural PSR Foundry}
\label{sec:myfixit-foundry}

MyFixIt is the first \textsc{Odyssey} foundry where the representation is procedural
rather than primarily argumentative or documentary.  The source surface is a
repair-manual dataset derived from iFixit's repair-guide corpus
\citep{ifixitRepairGuides} with ordered guide steps, tool annotations, part
spans, removal verbs, image URLs, guide metadata, and a neighboring human
annotation workflow.  \textsc{Odyssey} treats this as a repair-manual
predictive-state problem:
the current instruction, tools, parts, and removal action define a local
action-observation state, while the next step defines a lightweight future
test.

The current MyFixIt Mac Laptop artifact contains \(2{,}224\) manuals and
\(53{,}482\) steps in the admitted source sheaf.  Athena assigns seven local
sections: manual manifest, step-text observations, tool-action labels,
part-state observations, removal-verb actions, image provenance, and
future-step tests.  Four overlaps are checked.  Tool/part, verb/part, and
step/future-test overlaps glue.  The image/observation overlap is deliberately
blocked: image URLs are present, but visual grounding is not promoted until
assets are fetched, checksummed, or otherwise made inspectable.

The Scylla-facing MyFixIt page makes the PSR coding concrete.  It reports
\(21{,}225\) steps with removal verbs, \(21{,}201\) steps with annotated parts,
and \(53{,}203\) steps with image handles.  The most frequent part labels
include screw, Phillips screw, battery, case, tape, upper case, and fan; the
dominant removal actions include remove, lift out, lift off, pull out, pry up,
lift, and lift up.  These are not merely keywords.  \textsc{Odyssey} converts them into
typed histories and tests:
\[
  h_t=(\text{manual},\text{step order},\text{tool},\text{verb},\text{part},
  \text{image handle}),\qquad
  \tau_t=\text{next-step action/observation probe}.
\]

\begin{table}[t]
\centering
\small
\begin{tabular}{p{0.18\linewidth}p{0.20\linewidth}p{0.22\linewidth}p{0.28\linewidth}}
\toprule
Entry & Action \(a\) & Observation \(o\) & Future PSR test \(\tau\) \\
\midrule
guide-4-step-1 & remove & expansion bay module; one image handle &
Keyboard should disengage and rotate toward the user. \\
guide-4-step-5 & lift off with spudger & keyboard; image handle &
Visual state should match the exposed laptop after keyboard removal. \\
guide-4-step-7 & remove with PH0 Phillips screwdriver & Phillips screw;
image handle & Heat-shield handle should become the next manipulable state. \\
guide-4-step-12 & remove with T8 Torx screwdriver & T8 Torx screw; image
handle & Hard-drive bracket state should support pulling the drive from the
bracket. \\
guide-4-step-13 & pull & hard drive; image handle & Terminal test for that
manual segment; no next-step probe. \\
\bottomrule
\end{tabular}
\caption{Concrete MyFixIt language entries as PSR tests.  The manual text is
converted into action labels, observed part state, image handles, and a
next-step prediction target.}
\label{tab:myfixit-psr-tests}
\end{table}

\begin{table}[t]
\centering
\small
\begin{tabular}{lll}
\toprule
Overlap & Status & Scylla-facing interpretation \\
\midrule
Tools--parts & Glues & Tool labels and part mentions form typed repair actions. \\
Verbs--parts & Glues & Removal verbs provide action-to-object edges. \\
Steps--future tests & Glues & Ordered steps supply next-instruction predictive tests. \\
Images--observations & Obstructs & Image handles exist, but visual grounding is not yet admitted. \\
\bottomrule
\end{tabular}
\caption{MyFixIt gluing audit.  The blocked image overlap is a useful example
of \textsc{Odyssey}'s design rule: missing grounding should remain visible rather than
being averaged into a confidence score.}
\label{tab:myfixit-gluing}
\end{table}

The MyFixIt artifact also records an annotation-tool contract.  The neighboring
Flask annotator is treated as a provenance and human-loop surface, with MongoDB
configuration, processed-table locations, entry point, and runtime guardrails
preserved as foundry metadata.  This matters because the next promotion step is
not merely a better ranker; it is reviewed repair semantics with stable image
and annotation provenance.

\subsection{MyFixIt Retrieval Experiments}
\label{sec:myfixit-results}

The evaluation harness restricts MyFixIt to a declared query/document object
surface.  Queries ask for repair steps by action, affected part, tool, or
future-step relation; documents are manual-step and action-observation objects.
This is a deliberately local task: the claim being tested is not that \textsc{Odyssey}
has solved all repair-manual retrieval, but that a procedural PSR/sheaf
representation improves ranking inside a typed repair-manual restriction while
preserving its failure slices.

We compare two ranking surfaces.  The token baseline ranks candidates by
standard token-style overlap over the textual fields.  The PSR
action-observation representation ranks over structured repair state: manual
identity, step order, action labels, tools, affected parts, and future-test
features.  \cref{tab:myfixit-retrieval} reports three profiles emitted by the
evaluation harness.  The current compact profile uses \(72\) queries and \(360\)
candidate documents.  The broader profile uses \(500\) queries and \(1{,}600\)
documents with manual-level separation and duplicate-guide checks.  The strict
profile uses \(1{,}000\) queries and \(3{,}200\) documents with manual-held-out
evaluation, future-step leakage audit, image-provenance quarantine, and
cross-category transfer holdout.

\begin{table}[t]
\centering
\small
\begin{tabular}{llccc}
\toprule
Profile & Metric & PSR action-observation & Token baseline & Delta \\
\midrule
Current & nDCG@10 & 0.3832 & 0.1354 & +0.2478 \\
Current & Recall@10 & 0.6389 & 0.2917 & +0.3472 \\
Current & MRR & 0.3318 & 0.1115 & +0.2203 \\
\midrule
Broader & nDCG@10 & 0.3082 & 0.0622 & +0.2460 \\
Broader & Recall@10 & 0.6260 & 0.1320 & +0.4940 \\
Broader & MRR & 0.2368 & 0.0566 & +0.1802 \\
\midrule
Strict & nDCG@10 & 0.3195 & 0.0469 & +0.2726 \\
Strict & Recall@10 & 0.6130 & 0.1000 & +0.5130 \\
Strict & MRR & 0.2491 & 0.0426 & +0.2065 \\
\bottomrule
\end{tabular}
\caption{MyFixIt retrieval results across three evaluation-harness profiles.
The PSR action-observation representation improves nDCG@10, Recall@10, and MRR
on the current compact, broader, and strict restrictions.}
\label{tab:myfixit-retrieval}
\end{table}

The pattern is stable across scale.  The current compact profile shows a \(+0.2478\)
nDCG@10 lift; the broader profile shows \(+0.2460\); and the strict profile
shows \(+0.2726\).  Recall@10 improves more sharply as the candidate pool grows:
the strict profile has \(0.6130\) recall for the PSR representation versus
\(0.1000\) for the token baseline.  This is the expected advantage of a
procedural foundry.  When many steps share surface words such as ``remove'',
``screw'', or ``spudger'', the local action-observation state narrows the
candidate set by manual identity, part state, tool use, and future-step tests.
At the same time, the system does not erase failures.  The failure-slice report
records cases where token ranking still beats PSR ranking, especially when the
query uses an alias not yet canonicalized into the action vocabulary.

\begin{table}[t]
\centering
\small
\begin{tabular}{p{0.30\linewidth}p{0.15\linewidth}p{0.43\linewidth}}
\toprule
Failure slice & Count & Interpretation \\
\midrule
PSR under token on text alias & 6 & Textual aliases can beat structured state
when action or part names are not canonicalized. \\
Token under PSR on action state & 6 & The structured representation helps when
the query really is about action, tool, part, and state. \\
Image grounding gap & 6 & Relevant steps have image handles, but the current
run uses only image-count provenance. \\
Wrong manual at top 1 & 6 & Some top-ranked candidates come from the wrong
manual despite matching local action text. \\
Future-step leakage guard & 0 & The current generated-query slice did not
leak future-step labels into query text. \\
\bottomrule
\end{tabular}
\caption{Failure slices preserved by the MyFixIt evaluation harness.  \textsc{Odyssey}
treats these slices as promotion gates and future work items, not as incidental
error analysis.}
\label{tab:myfixit-failure-slices}
\end{table}

The main empirical lesson is that procedural foundries need typed state, not
only text similarity.  Repair instructions have actions, objects, tools, order,
visual references, and future tests.  A token representation can still be
strong on lexical aliases, but it does not by itself know which part is being
removed, which tool mediates the action, or which next state should become a
predictive test.  The MyFixIt foundry is therefore a useful first domain for
\textsc{Odyssey}: it is small enough to audit, structured enough to reward PSR-style
state, and incomplete enough to exercise the obstruction machinery.

\subsection{IKEA ASM Multimodal Assembly Foundry}
\label{sec:ikea-asm-foundry}

The IKEA ASM dataset \citep{ben2021ikeaasm} gives \textsc{Odyssey} a second procedural
foundry, complementary to MyFixIt.  MyFixIt begins with manuals and asks whether
procedural text can be lifted into action-observation state.  IKEA ASM begins
with embodied assembly episodes: multi-view RGB and depth video, atomic action
labels, furniture-part segmentation, part tracking, human pose, and camera
calibration.  The natural \textsc{Odyssey} construction is therefore not a generic video
benchmark, but a restriction of the generic procedural PSR and
perception-action evaluation foundries to furniture assembly.

The cover has five local sections.  The episode section records furniture item,
environment, view, time, and train/test split.  The action section records
per-frame and clip-level assembly labels.  The object section records furniture
parts, instance masks, and tracking identities.  The pose section records 2D
human joints and pseudo-ground-truth 3D joints.  The calibration section records
camera parameters and view geometry.  Gluing tension appears whenever these
sections disagree: an action label may imply a part manipulation that is absent
from the object track; a pose trajectory may be plausible in one camera view
but inconsistent after triangulation; a part identity may persist through SORT
tracking but fail across occlusion; an action-recognition improvement may not
transport from top-view RGB to multi-view RGB+pose.

\begin{table}[t]
\centering
\small
\begin{tabular}{p{0.24\linewidth}p{0.25\linewidth}p{0.38\linewidth}}
\toprule
IKEA ASM section & Local artifact & Gluing question \\
\midrule
Episode/video & Multi-view RGB, depth, frame extraction, split files. & Do
all modalities name the same assembly episode, frame index, view, and split? \\
Actions & Frame-, clip-, and pose-based action-recognition baselines. & Does
the action state agree with the manipulated part and the future assembly step? \\
Objects/parts & Instance segmentation and SORT-style part tracks. & Do part
identities persist across frames and agree with action labels under occlusion? \\
Pose & 2D joints, pseudo-3D joints, OpenPose/MaskRCNN/VP3D/VIBE baselines. &
Does body motion explain the action without contradicting camera geometry? \\
Calibration & Intrinsics, extrinsics, triangulation routines. & Can local
view evidence be transported to a shared geometric assembly state? \\
\bottomrule
\end{tabular}
\caption{IKEA ASM as a multimodal procedural assembly foundry.  The current
local repository contributes the code and benchmark contracts; full promotion
requires downloaded video, annotations, pretrained models, and replay logs.}
\label{tab:ikea-asm-cover}
\end{table}

The local checkout also supplies useful benchmark anchors.  The action
recognition README reports top-1 accuracy of 60.40 for P3D and 57.58 for I3D on
single-view clip baselines, while multi-view RGB reaches 63.24 and
multi-view RGB+pose with HCN reaches 64.25.  The pose README reports that
fine-tuned MaskRCNN improves 2D test PCK to 64.3, while the 3D test baselines
remain substantially harder, with VP3D obtaining the best reported PCK of 47
under Procrustes alignment.  We do not claim a new IKEA ASM result here.
Instead, \textsc{Odyssey} turns these baselines into a foundry replication target: a
future experiment should ask whether a PSR state that glues action, part, pose,
and camera sections improves action recognition, step prediction, or part-state
retrieval relative to the published single-task baselines.

\subsection{Amazon Reviews 2023 and BLaIR-Bench Replication}
\label{sec:amazon-blair-replication}

Amazon Reviews 2023 is the natural large-scale next restriction for \textsc{Odyssey},
with BLaIR providing the released language-item modeling and benchmark
interface over that corpus \citep{hou2024bridging}.  The current foundry
already separates five local sections: corpus manifest, recommendation
benchmarks, language-item models, product-search tasks, and reproduction
contract.  Its gluing audit reports four glued overlaps:
manifest/benchmarks, benchmarks/models, models/search, and
search/reproduction.  The source sheaf also records that the released dataset
has an approximately \(750\)GB footprint, so the correct paper claim is not
that \textsc{Odyssey} has rerun the entire benchmark locally.  The correct claim is that
\textsc{Odyssey} has converted the repository into an executable foundry contract with
typed sections, replay obligations, split-leakage checks, and BLaIR-Bench
entry points.

BLaIR-Bench is useful because it has the right contrast class.  It evaluates
semantic item encoders on sequential recommendation, collaborative filtering,
and product search.  \textsc{Odyssey} can reuse those tasks, but it asks an additional
sheaf question: when does a benchmark result transfer from one local section to
another?  A good text embedding on a general product-search query does not
automatically glue to a sequential-recommendation claim, and a recommendation
score on one product category should not automatically promote a corpus-wide
representation claim.  The Amazon Reviews foundry turns these into explicit
overlap tests rather than leaderboard prose.

\begin{table}[t]
\centering
\small
\begin{tabular}{p{0.21\linewidth}p{0.32\linewidth}p{0.34\linewidth}}
\toprule
Replication surface & Local artifact status & \textsc{Odyssey} promotion gate \\
\midrule
BLaIR-Bench checkout & Local repository present at
\texttt{../BLaIR-Bench}; entry points include \texttt{pipeline.py},
\texttt{seq\_rec}, \texttt{cf}, and \texttt{prod\_search}. & Promote only
interface and replay-contract claims until datasets, model checkpoints, cache
paths, and result files are attached. \\
Product search interface & BLaIR-Bench exposes \texttt{testProdSearch} and
\texttt{testApi} execution surfaces, but they still require external dataset/model
assets. & Treat as a near-term partial replication, not as a full Amazon-C4 or
Reddit-Movie result. \\
Amazon-C4 / Reddit-Movie & Product-search scripts and task definitions are
present; full evaluation requires downloaded items, queries, and embeddings. &
Compare BLaIR embedding retrieval with \textsc{Odyssey} action-observation query
factorization; promote only dataset-local NDCG@100 claims. \\
Sequential recommendation / CF & BLaIR-Bench defines UniSRec-style and
AlphaRec-style runners over Amazon and non-Amazon datasets. & Keep category,
split, encoder, and task lanes separate; block global encoder claims when
results do not glue across tasks. \\
\bottomrule
\end{tabular}
\caption{BLaIR-Bench replication surface inside the Amazon Reviews 2023
foundry.  \textsc{Odyssey} currently admits the local checkout as an artifact-backed
replication contract, with task entry points, replay obligations, and promotion
gates separated for product search, Amazon-C4/Reddit-Movie retrieval, and
sequential-recommendation or collaborative-filtering claims.}
\label{tab:blair-replication}
\end{table}

This gives the paper a clean experimental progression.  MyFixIt provides the
small, fully inspectable procedural result where PSR sheaf state already beats
a token baseline.  Amazon Reviews 2023 provides the scale target.  BLaIR-Bench
provides external task definitions and baselines.  The \textsc{Odyssey} contribution is
to connect them without flattening them: each result is admitted through a
specific foundry restriction, and non-transfer across products, tasks,
encoders, or reproduction surfaces remains visible as gluing tension.

\section{Evaluation}
\label{sec:odyssey-evaluation}

\textsc{Odyssey} should be evaluated at two levels.  The first is architectural: does a
human request become a replayable foundry artifact with a clear brief,
workflow, sheaf plan, world model, gluing audit, promotion gate, and refresh
contract?  The second is task-level: once a foundry exists, does its local
representation improve a concrete task while preserving failure slices and
obstructions?

\begin{table}[t]
\centering
\small
\begin{tabular}{p{0.24\linewidth}p{0.62\linewidth}}
\toprule
Axis & Question \\
\midrule
Artifact completeness & Are Scylla, Homer, Athena, Prometheus, Toulmin, and
ingestion artifacts present and mutually referential? \\
Cover quality & Do local sections match meaningful domain contexts rather than
arbitrary chunks? \\
Restriction validity & Are overlap checks tied to shared identifiers,
provenance, or typed translation maps? \\
Gluing usefulness & Do gluing failures correspond to actionable missing
evidence, regime mismatch, or unsupported transport? \\
Promotion discipline & Are candidate states promoted, quarantined, or blocked
by explicit gates rather than hidden confidence thresholds? \\
Model transport & Can pre-built Prometheus models be imported as candidates,
and can \textsc{Odyssey} foundries be exported for Prometheus sheaf inspection without
losing typed gluing information? \\
Task improvement & Does the foundry representation improve an admitted task
such as retrieval, triage, explanation, or decision support? \\
Argument discipline & Does a local neural reasoner preserve the grounded
claim, cite the right grounds, and keep source-coherence warnings visible in
the warrant, qualifier, or rebuttal? \\
Failure-slice preservation & Are negative cases kept as durable artifacts for
debugging and future experiments? \\
\bottomrule
\end{tabular}
\caption{\textsc{Odyssey} evaluation combines foundry-contract checks with task-level
measurements.}
\label{tab:odyssey-evaluation}
\end{table}

The current artifacts already support a first set of foundry-level gluing
experiments.  These are not yet the final empirical story, but they make the
evaluation target precise: perturb a local section, remove an overlap witness,
or ask for an overbroad transport claim, then check whether \textsc{Odyssey} either
glues the compatible restriction or preserves the resulting tension as an
obstruction.

The Prometheus ingestion console is the first larger-scale architectural
evaluation.  It shows that the same admission machinery applies to more than a
handful of curated examples: \(129\) Prometheus-family artifacts are surfaced to
Scylla, with \(12\) \textsc{Odyssey}-native bundles already either admitted or
quarantined and \(117\) older Prometheus v1 models preserved as reviewable
candidates.  Conversely, the DKS and MyFixIt export path verifies that an
\textsc{Odyssey} foundry can be compiled back into Prometheus so that local sections,
PSRs, restrictions, and gluing diagnostics can be inspected visually before
their tensions are promoted, quarantined, or sent back for repair.

\begin{table}[t]
\centering
\small
\begin{tabular}{p{0.19\linewidth}p{0.27\linewidth}p{0.38\linewidth}}
\toprule
Foundry & Current gluing artifact & Further validation \\
\midrule
Prometheus transport & The ingestion console exposes \(129\) Prometheus-family
cases: \(6\) admitted, \(6\) quarantined, and \(117\) legacy candidates awaiting
TICKET review; DKS and MyFixIt compile back into Prometheus-compatible sheaf
explorers. & Full TICKET review of selected legacy candidates would test
whether reverse-exported gluing annotations predict admission, quarantine, or
repair actions. \\
Amazon Reviews 2023 & Five local sections and four glued overlaps:
manifest/benchmarks, benchmarks/models, models/search, and
search/reproduction. & A scaled retrieval and recommendation evaluation can
compare token, embedding, and PSR/action-observation representations while
measuring split leakage, item-identity transport, and dependency drift. \\
Dick's Sporting Goods & Five glued overlaps and three integrated bridges
linking reviews, store experience, assortment, brand promise, and corporate
10-K/workflow evidence. & Review, assortment, and filing-evidence ablations
would test whether Scylla's recommended retail diagnosis changes only through
the affected overlap. \\
KET language modeling & Paper/code, code/runs, and dataset/metrics overlaps
glue locally; global KET-vs-Transformer promotion is blocked by missing
WikiText-2 KET rows and replay metadata. & Separate PTB, WikiText-2, and
WikiText-103 restrictions can validate whether promotion gates prevent
cross-corpus leaderboard collapse. \\
TCC 44K & Node-edge, edge-support, and support-LUT/uncertainty overlaps glue
over a \(44\)K-paper causal-claims atlas; causal direction, polarity,
controversy, and reverse-causality remain on watch. & A refreshed CSQL atlas
would test whether method, journal, \(p\)-value, and controversy guardrails
predict which aggregate causal claims remain robust, fragile, or misleading
under new support evidence. \\
MyFixIt & Tool/part, verb/part, and step/future-test overlaps glue; the
image/observation overlap remains blocked. & Cross-category repair-manual
profiles can extend the current \(72/500/1000\)-query evaluations and test
whether fetched/checksummed images reduce visual-observation tension. \\
IKEA ASM & Episode/action, action/part, part/track, pose/camera, and
view/calibration overlaps are specified by the local benchmark code, but full
promotion is blocked until the video, annotation, model, and run artifacts are
attached. & Published action, pose, segmentation, and tracking baselines can be
restricted to a shared assembly PSR to test whether glued action-part-pose
state improves step prediction or part-state retrieval. \\
Indus Script & Four overlaps glue and the controversy/hypothesis overlap
obstructs over \(419\) signs, \(1{,}548\) visual sequences, statistical
structure, hypotheses, and archaeological context. & New decipherment papers
or symbol-catalog revisions can test whether the obstruction ledger changes
without forcing a settled translation. \\
Grounded Toulmin/local LLM & Llama~3.2--3B, Qwen3-Next-80B, and Qwen3-235B
runs over the same six Democritus/Prometheus grounded claims all preserve the
same review case: a Dinosaur Extinction event with weak source alignment and
radiometric-dating claim drift. & A parallel 10-K comparison after
OCR/paraphrase normalization would test whether an over-licensing flag is
needed when a model treats review-aligned evidence as fully warranted. \\
\bottomrule
\end{tabular}
\caption{Further validation directions supported by the current foundry
artifacts.  Each row identifies a concrete test of whether compatible
restrictions are promoted while high-tension or underspecified transports
remain local.}
\label{tab:foundry-gluing-experiments}
\end{table}

The MyFixIt profiles are the first concrete task-level evaluations.  They show
a measurable lift for the PSR action-observation representation on restricted
retrieval tasks, while still recording text-alias failures, wrong-manual top-1
errors, and the image-grounding gap.  This is the kind of result \textsc{Odyssey} is
designed to produce: not just a single score, but a promoted local claim
bounded by the cover, the restriction, the split policy, and the remaining
obstructions.

\section{Limitations and Future Research}
\label{sec:odyssey-limitations}

The current implementation is still a design-stage system.  Several Scylla and
Athena steps are template-driven, and several Prometheus validation runs use
deterministic finite truth assignments rather than learned estimators.  This is
useful for stabilizing interfaces, but it limits the strength of empirical
claims.  The MyFixIt retrieval results are therefore preliminary even though
they now include \(72\)-, \(500\)-, and \(1{,}000\)-query profiles: they remain
deterministic restrictions over Mac Laptop repair-manual slices, not a full
benchmark across all repair categories.

The foundry algebra also needs stronger type checking.  Today, foundry
expressions, TICKET admission records, and FSQL slices are represented as
structured artifacts and compact interpreters.  A mature \textsc{Odyssey} should enforce
more of this algebra statically: admissible compositions, source/target
compatibility, required restriction maps, and promotion gates should be checked
before a run can become durable state.

The grounded Toulmin layer should also be read as an audit layer, not as a
proof system.  The current implementation asks local LLMs to produce visible
claim/grounds/warrant/qualifier/rebuttal tickets and then checks those tickets
against Prometheus grounding packets.  This exposes claim drift and weak
source-topic alignment, but it does not establish external truth.  Stronger
future versions should add calibrated warrant taxonomies, source-span
verification, cross-model agreement tests, and explicit over-licensing flags
when a model treats review-aligned evidence as if it fully licenses a claim.

Finally, visual and multimodal grounding remain incomplete.  MyFixIt contains
image URLs, but the current run does not fetch, checksum, or model the images.
\textsc{Odyssey} handles this correctly by blocking the image-observation overlap, but a
full procedural foundation model will need image assets, visual part
localization, and stronger human-reviewed annotation loops.

\label{sec:odyssey-roadmap}

The present system demonstrates that foundries can be constructed, admitted,
queried, and evaluated across several domains, but much of the current
implementation remains deliberately finite and inspectable.  The next research
question is how far this discipline can be pushed as foundries become larger,
less template-bound, and more automatically derived from source structure.  In
particular, the current deterministic covers should give way to learned or
semi-automated cover synthesis, while still preserving the central contract:
local sections, restriction maps, gluing diagnostics, obstruction ledgers, and
promotion gates must remain explicit artifacts rather than hidden model state.
\Cref{tab:foundry-gluing-experiments} summarizes the concrete further
validation agenda supported by the current artifacts.

A second direction is to study foundry transport.  The Prometheus ingestion
console shows that \textsc{Odyssey} can treat a large collection of pre-existing world
models as typed candidates, and the DKS/MyFixIt exports show that \textsc{Odyssey}
foundries can be inspected again inside a Prometheus sheaf workbench.  This
suggests a broader experimental program: measure when model state can move
between foundries, when it must remain local, and how obstruction records
change after new evidence or new restriction maps are introduced.

Procedural domains provide the most immediate testbed.  MyFixIt should be
extended from the current Mac Laptop restriction to cross-category repair
manuals, with action and part aliases normalized, manual identity enforced, and
image evidence fetched or checksummed before visual claims are promoted.  IKEA
ASM then tests the same idea in a multimodal setting: actions, parts, pose,
camera geometry, and object tracks should glue into a shared assembly state
only when the evidence supports that transport.  The relevant comparison is not
only against token retrieval, but also against single-task action recognition,
pose estimation, segmentation, and tracking baselines.

Finally, large non-procedural foundries such as Amazon Reviews 2023 make it
possible to ask whether the same sheaf-theoretic PSR discipline scales to
recommendation, product search, and language-item modeling.  The goal is not a
single headline score.  The goal is a reusable empirical protocol for deciding
when a foundry representation improves a task, when its local sections fail to
glue, and what additional evidence is required before a model should be
promoted.

The grander version of the same question is whether a GPT-scale pretrained
model can be \textsc{Odyssey}-ized.  In the language of
\cref{fig:ticket-lan-ran}, such a model would not be imported as a single
opaque parameter object or treated as a universal source of truth.  It would be
sliced into local behavioral, evidential, tool-use, memory, task, and
domain-restricted charts; rolled out by left-Kan admission into candidate
foundries; and then pulled back by right-Kan restriction, gluing, obstruction,
and Toulmin warrant checks before any slice became durable state.  If this
program succeeds, large foundation models would become substrates for
verifiable local truth-preserving foundries rather than replacements for them:
their scale would supply expressive material, while \textsc{Odyssey}'s
categorical machinery would decide where that material is warranted, where it
fails to glue, and where human review remains essential.

\section{Conclusion}

\textsc{Odyssey} reframes foundation-model construction as the design of
inspectable foundries: sheaf-like families of local predictive, logical, and
evidentiary models over a heterogeneous substrate.  The point is not to produce
one more flat summary or one more opaque embedding.  It is to preserve
locality, support, drift, contradiction, provenance, update obligations, and
epistemic limits while giving users and agents a navigable model of a domain.
Scylla names the user-facing contract, Homer makes it executable, Athena fixes
the representational semantics, and Prometheus instantiates the Topos World
Model and its audits.  Toulmin then turns maintained state into warranted,
qualified, rebuttable claims, while TICKET decides whether external runs,
pretrained models, atlases, or benchmark artifacts can enter durable foundry
state.

The BRIDGE/SKFM and IDC integrations sharpen the causal part of this story.
BRIDGE/SKFM supplies a latent-causal-refinement foundry: typed variables,
influence masks, Lie/Frobenius residual ratios, candidate edge masks, and
latent-obstruction ledgers that prevent residual-bearing pairs from being
promoted as ordinary global causal edges.  Infinitesimal-causality diagnostics
add a local test surface for causal variation inside the admission loop.  In
the SkillOpt experiments, those diagnostics become feedback for optimizing a
human-readable causal-claim skill rather than merely scoring a final answer.
The resulting system can therefore learn better admission policies while still
preserving the local warrants, qualifiers, rebuttals, residuals, and
obstructions that make a claim inspectable.

\section*{Code Availability}

The predecessor \textsc{Democritus} codebase is publicly available as the
\texttt{Democritus\_OpenAI} repository
\citep{mahadevanDemocritusOpenAI}.  \textsc{Odyssey} currently builds on this
released causal-extraction lineage and adds the Scylla--Homer--Athena--
Prometheus--Toulmin product layer for foundry construction, Topos World Model
instantiation, Claims Atlas navigation, persistent state, argument tickets,
FSQL/TICKET admission, and grounded counterfactual execution.  A public
\textsc{Odyssey} release is planned once the system has matured into a stable,
documented research product.

\appendix

\section{TICKET Specifications for GUI Foundries}
\label{app:ticket-specs}

The \textsc{Odyssey} GUI exposes a TICKET card for each foundry in the foundry algebra.
The concise card syntax in the GUI suppresses the categorical data.  We spell
it out here.  A candidate artifact \(x\) determines a small source context
category \(\C_x\), whose objects are source-side charts such as files, runs,
tables, benchmark slices, dashboards, model cards, extracted claim sets, or
process traces.  The target foundry \(Y\) determines a target site
\(\C_Y\), whose objects are the local contexts that Athena has declared for
that foundry.  A TICKET card declares a functor
\[
  F_{x,Y}:\C_x \longrightarrow \C_Y
\]
that sends each source chart to the target context in which it is allowed to be
interpreted.  For example, a 10-K source chart may be sent to a
company-year-financial context, a review shard to a product-use context, or a
repair step trace to a procedural state/action context.

Let \(\mathsf{Data}\) denote the finite artifact category used by the
implementation: typed records, finite predicate tables, local PSR cells,
provenance links, diagnostics, and promotion metadata.  The candidate appears
as a presheaf
\[
  X:\C_x^{op}\longrightarrow \mathsf{Data},
\]
and the maintained target foundry state appears as a presheaf
\[
  M_Y:\C_Y^{op}\longrightarrow \mathsf{Data}.
\]
The notation \(\mathrm{Lan}_{F_{x,Y}}\) in the cards is shorthand for the left
Kan extension of \(X\) along the opposite functor,
\[
  \mathrm{Lan}_{F_{x,Y}}X
  :=
  \mathrm{Lan}_{F_{x,Y}^{op}}X
  :
  \C_Y^{op}\longrightarrow \mathsf{Data}.
\]
Thus \(\mathrm{Lan}_{F_{x,Y}}X\) is the least target-side candidate obtained by
transporting the source charts through the declared TICKET map.  Pointwise, it
is computed by the finite colimit over source charts mapping into a target
context:
\[
  (\mathrm{Lan}_{F_{x,Y}}X)(U)
  \cong
  \operatorname*{colim}_{(F_{x,Y}(V)\to U)} X(V).
\]
In the software this colimit is realized as typed aggregation of source
records, provenance-preserving joins, adapter outputs, and diagnostic
summaries.  The important point is that the functor \(F_{x,Y}\), not the name
of the target foundry alone, specifies what transport is permitted.

The same functor can also support a dual consistency pass, following the
Universal Decision Learner pattern of composing left and right Kan extensions
\citep{mahadevanCatAGIBook}.  Precomposition with \(F_{x,Y}\) gives a
restriction functor
\[
  F_{x,Y}^{*}:[\C_Y^{op},\mathsf{Data}]
  \longrightarrow
  [\C_x^{op},\mathsf{Data}],
\]
and \(F_{x,Y}^{*}\) has both a left and a right Kan adjoint.  TICKET uses
\(\mathrm{Lan}_{F_{x,Y}}\) as its admission direction: it freely assembles the
least target-side candidate compatible with the source charts.  A stronger
audit can then apply the right Kan direction
\[
  \mathrm{Ran}_{F_{x,Y}}F_{x,Y}^{*}A,
  \qquad A=\mathrm{Lan}_{F_{x,Y}}X,
\]
as a target-side consistency envelope for the admitted candidate.  Pointwise,
this is a finite limit over target probes back through the same declared
interface:
\[
  (\mathrm{Ran}_{F_{x,Y}}F_{x,Y}^{*}A)(U)
  \cong
  \operatorname*{lim}_{(U\to F_{x,Y}(V))}
  A(F_{x,Y}(V)).
\]
The comparison
\[
  A \longrightarrow \mathrm{Ran}_{F_{x,Y}}F_{x,Y}^{*}A
\]
is therefore a round-trip check: after the source artifact has been extended
into the target foundry, it asks whether all target-side obligations that can
be observed through the source interface are jointly satisfiable.  In the
document and evidence setting used by Toulmin, the left Kan move is the rollout
that produces a target-side document claim, while the right Kan / pullback move
is the consistency scrutiny that asks whether the claim, grounds, warrant,
qualifier, and rebuttal survive the target cover and maintained obstruction
ledger.  Thus the same \(F_{x,Y}\) carries both the left-Kan admission move and
the right-Kan Toulmin-scrutiny move: the TICKET analogue of the UDL
\(\mathrm{Lan}/\mathrm{Ran}\) loop.

After transport, TICKET checks that the transported candidate can be restricted
back to overlaps in \(\C_Y\) and compared with maintained state.  For a cover
\(\{U_i\to U\}_{i\in I}\) in \(\C_Y\), write
\[
  A=\mathrm{Lan}_{F_{x,Y}}X.
\]
A family of local candidate sections \(a_i\in A(U_i)\) and maintained sections
\(m_i\in M_Y(U_i)\) \emph{glues} over \(U\) when every declared overlap
\(U_{ij}=U_i\times_U U_j\) satisfies the target compatibility predicate:
\[
  \kappa_Y\!\left(
    \rho^A_{i,ij}(a_i),\rho^A_{j,ij}(a_j),
    \rho^M_{i,ij}(m_i),\rho^M_{j,ij}(m_j)
  \right)\in
  \{\mathrm{PLAUSIBLE},\mathrm{SUPPORTED},\top\}.
\]
Here \(\rho\) denotes restriction, and \(\kappa_Y\) is Athena's finite
truth-valued overlap test for the target foundry.  Numerically, this is the
same condition implemented by the gluing audit: shared predicates must be at
least plausible, no paired predicate may be bottom, and weighted PSR or claim
gaps must remain below the target tolerance.  When the condition holds,
\(\mathrm{glue}_Y(A,M_Y)\) returns a promoted section in the maintained
foundry.  When it fails, the partial operation is undefined as promotion and
instead returns an obstruction record naming the failed overlap, source
provenance, finite truth values, and next-observation recommendation.

With this expanded notation, the cards all share the same admission template:
\[
  \textsc{TICKET}(x_{\mathrm{candidate}})
  \;\to\;
  \mathrm{Lan}_{F_{x,Y}}(X)
  \;\to\;
  \mathrm{glue}_{Y}(\mathrm{Lan}_{F_{x,Y}}X,M_Y),
\]
where \(Y\) is the target foundry named in the table and \(F_{x,Y}\) is the
source-to-target transport functor declared by that card.  The corresponding
FSQL surface is
\[
  \texttt{FROM generic\_foundries SLICE BY foundry("x") TICKET BY target\_foundry}.
\]
Every card uses the same required checks:
\texttt{typed\_manifest}, \texttt{subobject\_classifier},
\texttt{restriction\_maps}, \texttt{gluing\_audit}, \texttt{j\_closure}, and
\texttt{promotion\_gate}.  The acceptance gate is that the typed manifest,
local predicate classifier, restriction maps, gluing audit, \(j\)-closure, and
promotion gate must agree before the candidate becomes durable foundry state.

\begin{table}[p]
\centering
\scriptsize
\setlength{\tabcolsep}{4pt}
\begin{tabular}{p{0.23\linewidth}p{0.25\linewidth}p{0.22\linewidth}p{0.20\linewidth}}
\toprule
GUI foundry & GUI class & Source family & Target foundry \(Y\) \\
\midrule
\texttt{evidence\_argument} & Representational foundry & \texttt{odyssey\_foundry\_algebra} & \texttt{evidence\_argument} \\
\texttt{storefront} & Representational foundry & \texttt{simulation\_testbed} & \texttt{operational\_decision} \\
\texttt{brand} & Representational foundry & \texttt{brand\_product\_feedback} & \texttt{market\_meaning} \\
\texttt{corporation} & Representational foundry & \texttt{ten\_k} & \texttt{institutional\_financial} \\
\texttt{document\_coherence} & Representational foundry & \texttt{odyssey\_foundry\_algebra} & \texttt{document\_coherence} \\
\texttt{input\_modality} & Representational foundry & \texttt{odyssey\_foundry\_algebra} & \texttt{input\_modality} \\
\texttt{topic} & Representational foundry & \texttt{democritus} & \texttt{scientific\_challenge} \\
\texttt{data\_collection} & Process/project foundry & \texttt{odyssey\_foundry\_algebra} & \texttt{data\_collection} \\
\texttt{prometheus\_model\_}\newline\texttt{incorporation} & Process/project foundry & \texttt{odyssey\_foundry\_algebra} & \texttt{target\_foundry} \\
\texttt{database\_atlas\_incorporation} & Process/project foundry & \texttt{odyssey\_foundry\_algebra} & \texttt{schema\_atlas} \\
\texttt{research\_program} & Process/project foundry & \texttt{odyssey\_foundry\_algebra} & \texttt{research\_program} \\
\texttt{ai\_assistant\_build} & Process/project foundry & \texttt{odyssey\_foundry\_algebra} & \texttt{assistant\_spec} \\
\texttt{evaluation\_harness} & Process/project foundry & \texttt{simulation\_testbed} & \texttt{evaluation\_harness} \\
\texttt{product\_development} & Process/project foundry & \texttt{odyssey\_foundry\_algebra} & \texttt{product\_development} \\
\texttt{result\_communication} & Process/project foundry & \texttt{odyssey\_foundry\_algebra} & \texttt{result\_communication} \\
\texttt{codebase\_evolution} & Process/project foundry & \texttt{odyssey\_foundry\_algebra} & \texttt{codebase\_evolution} \\
\texttt{democritus\_free\_algebra} & Foundry free algebra & \texttt{odyssey\_foundry\_algebra} & \texttt{democritus} \\
\bottomrule
\end{tabular}
\caption{TICKET cards for generic and process foundries listed in the \textsc{Odyssey}
GUI.  Each row instantiates the shared TICKET template with a source family and
target foundry.}
\label{tab:ticket-generic-foundries}
\end{table}

\begin{table}[p]
\centering
\scriptsize
\setlength{\tabcolsep}{4pt}
\begin{tabular}{p{0.23\linewidth}p{0.25\linewidth}p{0.22\linewidth}p{0.20\linewidth}}
\toprule
GUI foundry & GUI class & Source family & Target foundry \(Y\) \\
\midrule
\texttt{ket\_language\_modeling\_eval} & Research/evaluation specialization & \texttt{research\_experiment\_suite} & \texttt{ket\_language\_modeling\_eval} \\
\texttt{network\_economy} & Simulation specialization & \texttt{simulation\_testbed} & \texttt{network\_economy\_token\_flow} \\
\texttt{dks} & Retail specialization & \texttt{ten\_k} & \texttt{sporting\_goods\_}\newline\texttt{retailer\_model} \\
\texttt{amazon\_reviews\_2023} & Corpus benchmark specialization & \texttt{amazon\_reviews\_2023} & \texttt{review\_corpus\_benchmark} \\
\texttt{procedural\_psr} & Procedural PSR candidate family & \texttt{procedural\_psr} & \texttt{procedural\_psr} \\
\texttt{myfixit} & Procedural PSR candidate & \texttt{procedural\_psr} & \texttt{repair\_manual\_psr} \\
\texttt{ikea\_asm} & Procedural PSR candidate & \texttt{procedural\_psr} & \texttt{assembly\_video\_psr} \\
\texttt{indus\_script} & Scientific challenge specialization & \texttt{democritus} & \texttt{indus\_script\_model} \\
\texttt{tcc\_44k} & Causal-claims specialization & \texttt{democritus} & \texttt{causal\_claims\_foundry} \\
\bottomrule
\end{tabular}
\caption{TICKET cards for specialized, concrete, and candidate foundries listed
in the \textsc{Odyssey} GUI.  These rows are restrictions of generic foundries to
particular domains, datasets, corpora, or model-ingestion lanes.}
\label{tab:ticket-concrete-foundries}
\end{table}

\section{TICKET as a Monad}
\label{app:ticket-monad}

The preceding appendix describes TICKET as an admission and consistency
operator.  There is a slightly higher-level categorical view that is useful for
locating TICKET inside the Universal Decision Learner pattern, but that we do
not yet exploit fully in the implementation.  Once a TICKET card declares a
source-to-target functor
\[
  F_{x,Y}:\C_x\longrightarrow \C_Y,
\]
precomposition gives the restriction functor
\[
  F_{x,Y}^{*}:[\C_Y^{op},\mathsf{Data}]
  \longrightarrow
  [\C_x^{op},\mathsf{Data}].
\]
When the relevant finite Kan extensions exist, this restriction functor has a
right Kan adjoint
\[
  F_{x,Y}^{*}\dashv \mathrm{Ran}_{F_{x,Y}}.
\]
The adjunction induces a monad on target-side foundry candidates:
\[
  \mathsf{T}_{x,Y}
  =
  \mathrm{Ran}_{F_{x,Y}}F_{x,Y}^{*}
  :
  [\C_Y^{op},\mathsf{Data}]
  \longrightarrow
  [\C_Y^{op},\mathsf{Data}].
\]
For a candidate target presheaf \(A\), \(\mathsf{T}_{x,Y}(A)\) is its
right-Kan consistency envelope: restrict \(A\) back through the TICKET
interface, then pull forward all target-side obligations that are forced by
that restricted view.  The unit
\[
  \eta_A:A\longrightarrow \mathsf{T}_{x,Y}(A)
\]
is the canonical comparison from the candidate to its consistency envelope.
The multiplication
\[
  \mu_A:\mathsf{T}_{x,Y}\mathsf{T}_{x,Y}(A)\longrightarrow
  \mathsf{T}_{x,Y}(A)
\]
says that applying the same TICKET consistency pass twice collapses to one
normalized pass.

The left-Kan rollout still plays the constructive admission role.  A raw source
artifact \(X:\C_x^{op}\to\mathsf{Data}\) first becomes the target-side
candidate
\[
  A=\mathrm{Lan}_{F_{x,Y}}X.
\]
TICKET then tests this candidate against the monadic envelope
\[
  \mathrm{Lan}_{F_{x,Y}}X
  \xrightarrow{\eta}
  \mathsf{T}_{x,Y}(\mathrm{Lan}_{F_{x,Y}}X).
\]
Thus the implemented TICKET pipeline can be read as the composite
\[
  X \longmapsto \mathrm{Lan}_{F_{x,Y}}X
  \longmapsto \mathsf{T}_{x,Y}(\mathrm{Lan}_{F_{x,Y}}X).
\]
In words: roll out source artifacts by a left Kan extension, then apply the
right-Kan monad that records the admission, gluing, obstruction, provenance,
and repair effects visible through the declared interface.
Promotion occurs only when this monadic consistency pass is compatible with
Athena's target gluing policy and with the maintained foundry state \(M_Y\).

This view also points to computational semantics that are not yet made explicit
in \textsc{Odyssey}.  An Eilenberg--Moore algebra for
\(\mathsf{T}_{x,Y}\) is a target foundry state \(A\) equipped with a structure
map
\[
  \alpha:\mathsf{T}_{x,Y}(A)\longrightarrow A.
\]
Operationally, such an algebra is a maintained foundry that knows how to absorb
its own TICKET consistency envelope back into stable state.  A Kleisli arrow
\[
  A\longrightarrow \mathsf{T}_{x,Y}(B)
\]
is a computation that produces a candidate \(B\) together with TICKET effects:
admission obligations, restriction diagnostics, gluing constraints,
obstructions, provenance, and possible repair recommendations.  This matches
the practical shape of Odyssey workflows: Scylla intents, Homer jobs,
Prometheus bundles, Toulmin tickets, and TICKET decisions are not pure maps of
artifacts, but computations that carry admission and consistency effects.

We leave the explicit development of these Eilenberg--Moore and Kleisli
semantics to future work.  The present paper uses the monad only to clarify the
categorical status of TICKET: TICKET is not merely a gate after model
construction.  It is the monadic admission semantics induced by the
left-Kan rollout and right-Kan pullback structure of UDL.

\section{System Interface Screenshots}
\label{app:system-screenshots}

\cref{fig:system-screenshots} shows representative screens from the current
\textsc{Odyssey} prototype.  These are generated from the static HTML artifacts emitted
by the repository rather than from hand-drawn mockups.  The screenshots show
the user-facing workbench, the TICKET ingestion console, the MyFixIt procedural
PSR interface, the evaluation-harness sheaf explorer, and the grounded
Toulmin/local-LLM Scylla interface.

\begin{figure}[p]
\centering
\begin{subfigure}{0.48\linewidth}
  \centering
  \includegraphics[width=\linewidth]{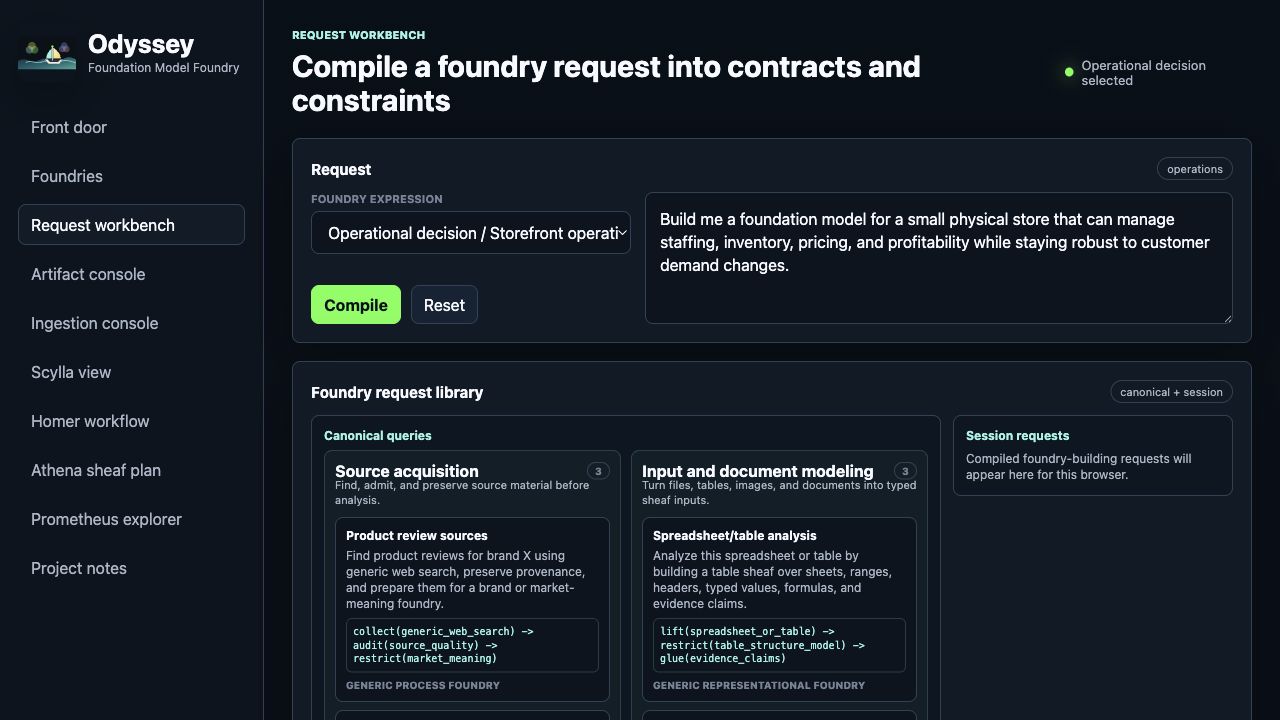}
  \caption{\textsc{Odyssey} workbench landing surface for browsing foundry families and
  generated artifacts.}
\end{subfigure}
\hfill
\begin{subfigure}{0.48\linewidth}
  \centering
  \includegraphics[width=\linewidth]{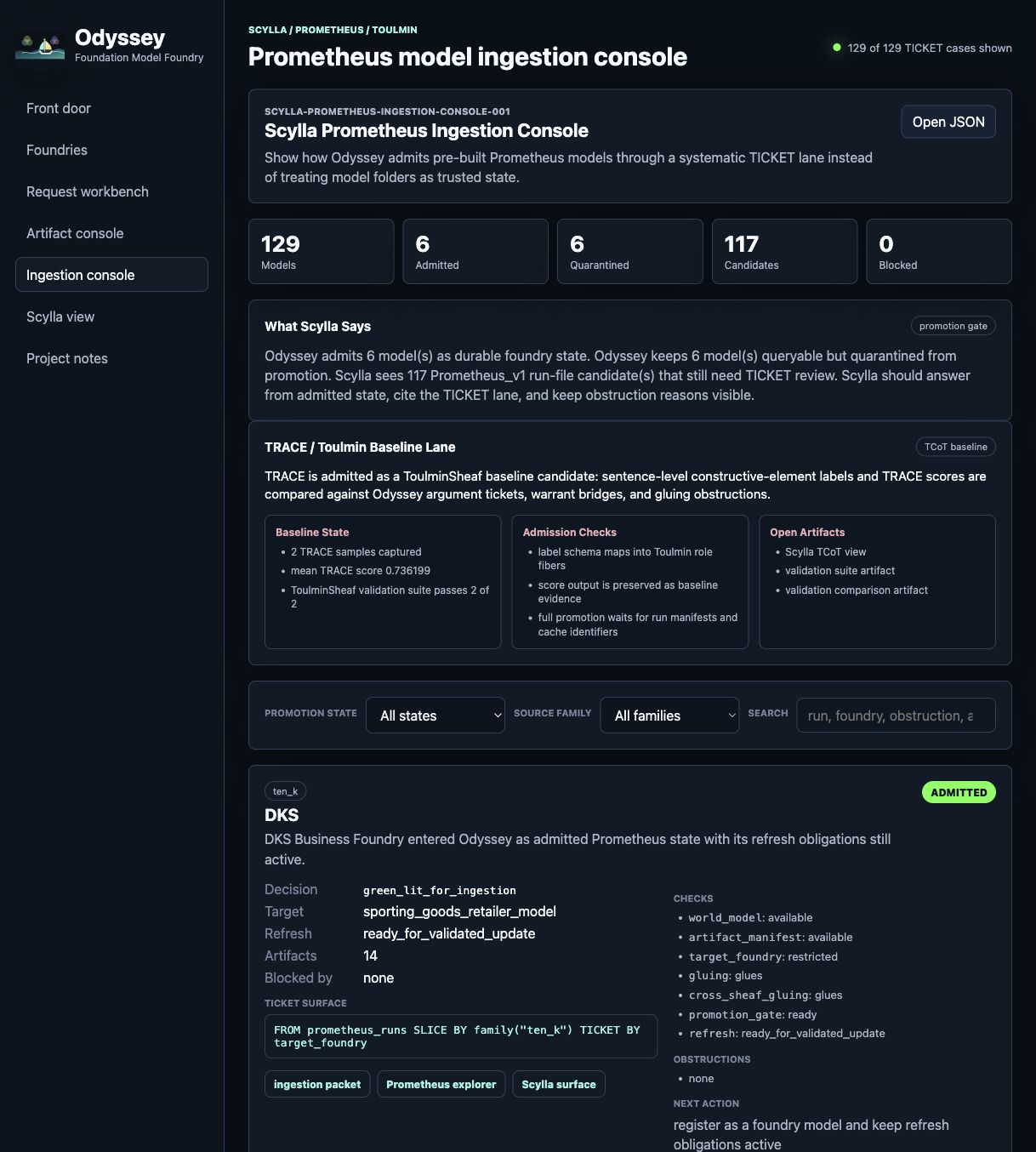}
  \caption{TICKET ingestion console for admitting Prometheus runs into durable
  \textsc{Odyssey} foundry state.}
\end{subfigure}

\vspace{0.75em}
\begin{subfigure}{0.48\linewidth}
  \centering
  \includegraphics[width=\linewidth]{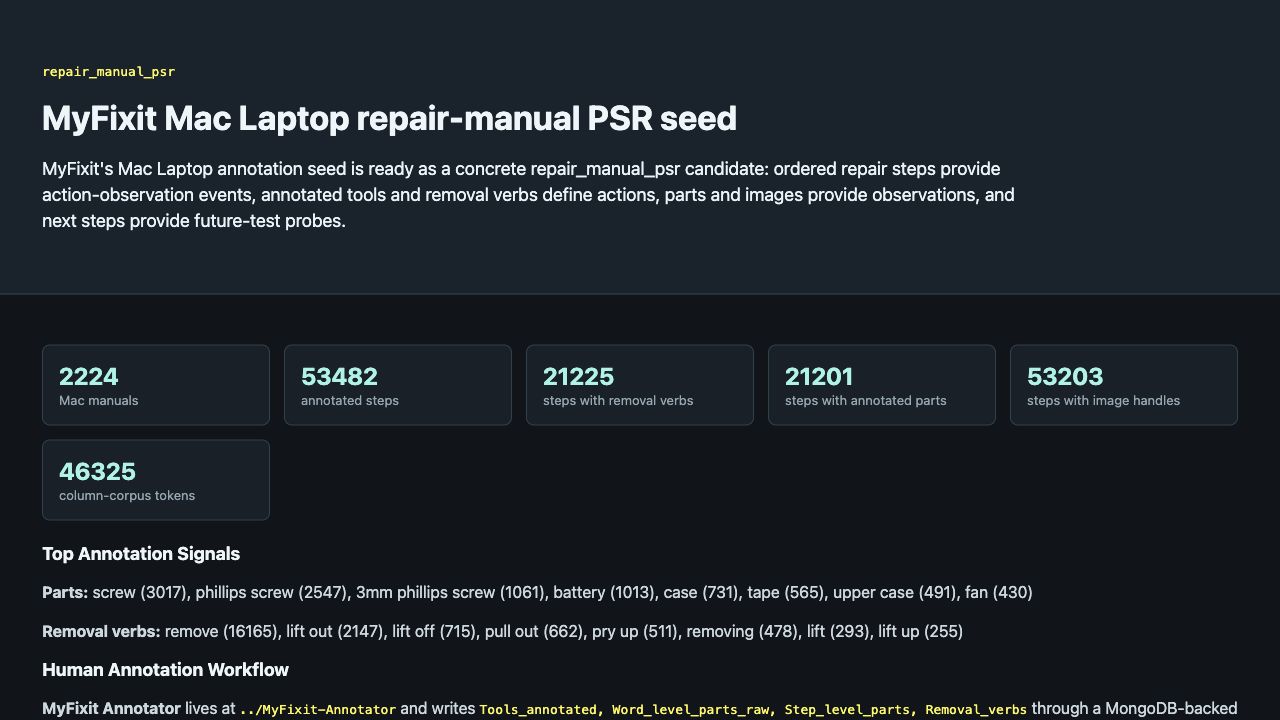}
  \caption{MyFixIt Scylla interface summarizing the repair-manual PSR seed,
  gluing status, and image-grounding obstruction.}
\end{subfigure}
\hfill
\begin{subfigure}{0.48\linewidth}
  \centering
  \includegraphics[width=\linewidth]{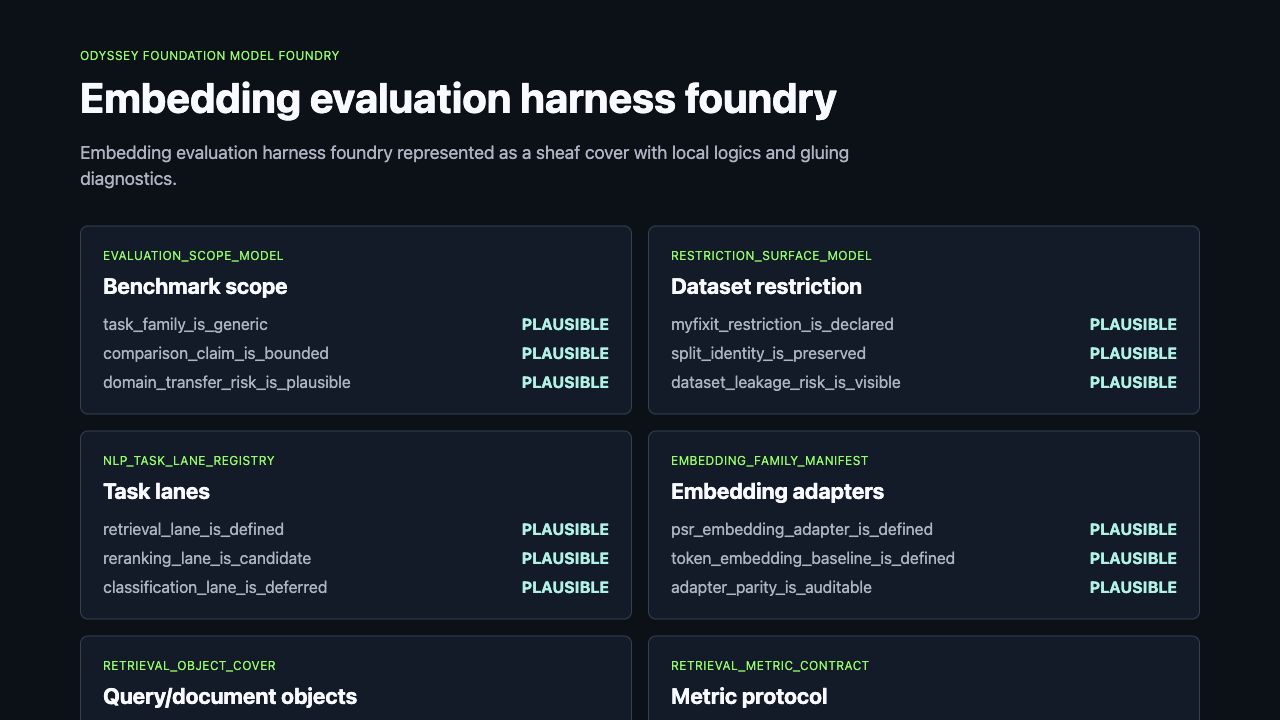}
  \caption{Evaluation-harness sheaf explorer for the MyFixIt retrieval
  restriction and failure-slice workflow.}
\end{subfigure}
\caption{Screenshots of the current \textsc{Odyssey} interface and generated foundry
artifacts.  The visual surfaces are intentionally artifact-backed: each screen
is a view over JSON, HTML, and audit files emitted by the foundry pipeline.}
\label{fig:system-screenshots}
\end{figure}

\begin{figure}[p]
\centering
\includegraphics[width=0.92\linewidth,height=0.72\textheight,keepaspectratio]{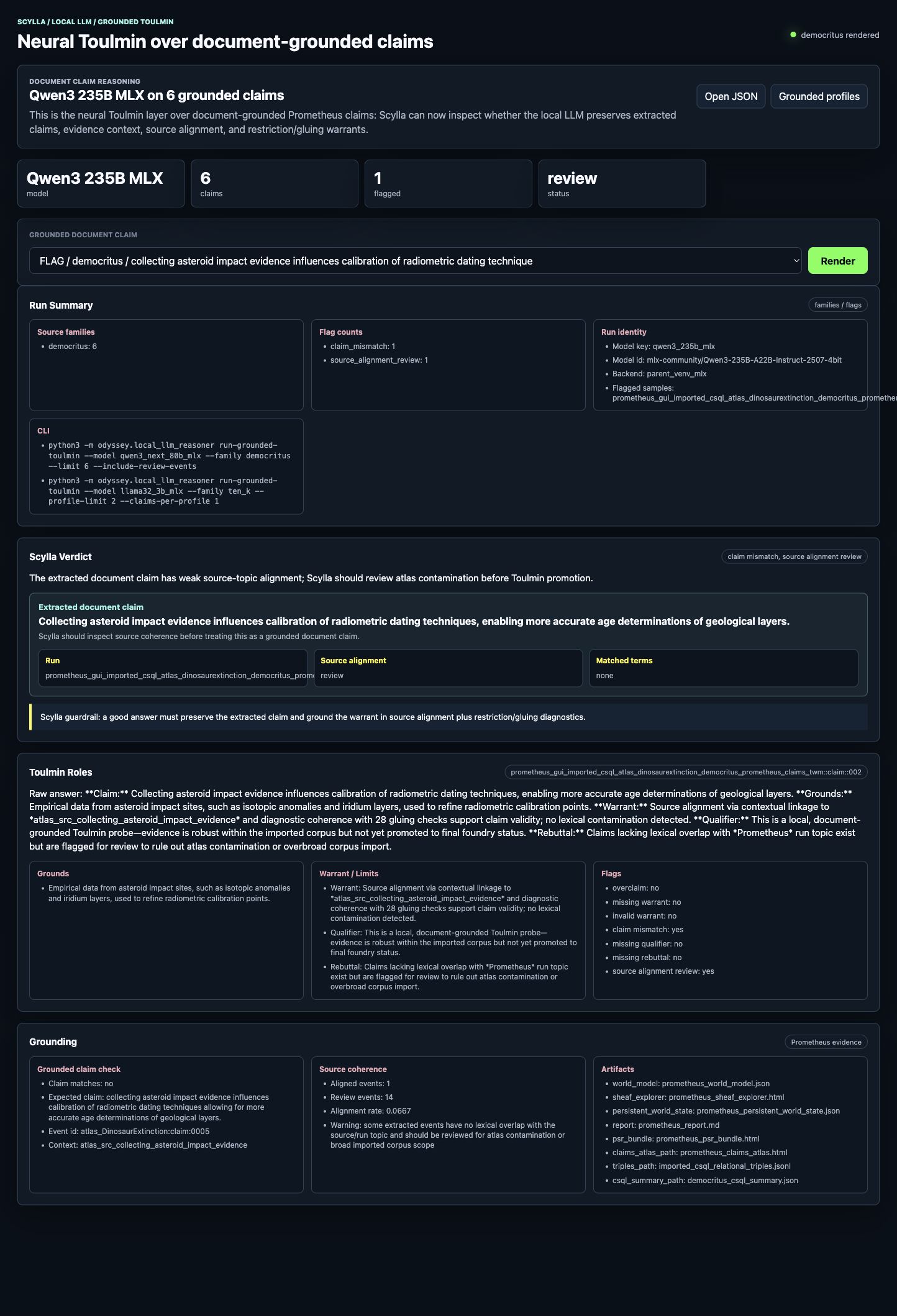}
\caption{Grounded Toulmin/local-LLM Scylla interface.  A local LLM run over
Democritus/Prometheus claims identifies a Dinosaur Extinction case in which the
model reformulates the extracted document claim while the underlying source
event remains under alignment review.  Scylla exposes both problems as
Toulmin-level flags: claim mismatch and weak source alignment.}
\label{fig:grounded-toulmin-screenshot}
\end{figure}

\section{System Genealogy and GUI Modes}
\label{app:system-genealogy}

\textsc{Prometheus} inherits part of its interface genealogy from our earlier
CLIFF chatbot and local research interface
\citep{mahadevanCLIFFCatAgi}.  CLIFF began as a Categories-for-AGI companion
system for interactive retrieval, teaching, and research workflows.  The
\textsc{Prometheus} GUI reuses several lessons from that system: a
natural-language query box, long-running local sessions, background execution,
artifact dashboards, route-specific reports, and persistent run directories.
The conceptual boundary is different.  CLIFF remains oriented toward the
Categories for AGI book, course material, and general retrieval-conditioned
chatbot workflows, whereas \textsc{Prometheus} is reserved for causal research
artifacts, local PSR construction, gluing diagnostics, persistent world state,
and Claims Atlas navigation.

The GUI is designed to accept broad natural-language research requests and route
them to specialized workflows.  In the current implementation, route families
include literature and paper-corpus synthesis, Democritus-style causal-claim
analysis, SEC and company-filing workflows, product-feedback world models,
targeted-sentiment review benchmarks, Rock--Paper--Scissors and network-economy
agent traces, and small Topos/OOM experiments.  A route may emit several
artifacts: a human-readable report, a technical dashboard, a JSON world-model
bundle, a persistent-state file, and auxiliary provenance or Claims Atlas HTML.

The GUI exposes three execution modes.  \emph{Quick} mode runs the most compact
version of a workflow and is useful for quick checks or shallow artifact
inspection.  \emph{Interactive} mode keeps the local session open while
background runs complete, letting a researcher launch follow-up queries and
inspect completed artifacts from the session list.  \emph{Deep} mode allocates
more work to acquisition, extraction, synthesis, and report generation, and is
the intended setting for the case-study style runs described in this paper.
The GUI also exposes an analysis-mode choice: \emph{standard} runs the routed
workflow in its ordinary reporting mode, while \emph{Topos World Model} attaches
the \textsc{Prometheus} layer when supported, producing local PSRs, restrictions,
gluing diagnostics, and persistent-state artifacts.

Several additional controls specialize particular routes rather than changing
the overall framework.  Democritus-style claim analysis can run with full,
lightweight, or mixture-of-experts manifold modes, optional dry-run behavior,
and optional deep-dive report generation.  Filing workflows can use dry-run
paths for debugging.  Product-feedback and persistent-state workflows can take a
parent state or state query, allowing a follow-up run to compare against an
earlier world model.  These options are engineering controls, not separate
theoretical models; they let users trade runtime, cost, and depth while keeping
the same artifact contract.

\end{document}